%% file: acml24_submission_template.tex
\documentclass[wcp]{jmlr}

\usepackage{longtable}

\usepackage{booktabs}

\usepackage{lineno}
\usepackage{booktabs}
\usepackage{algorithm}
\PassOptionsToPackage{linesnumbered,vlined,algo2e,ruled}{algorithm2e}
\usepackage{algorithm2e}
\usepackage{graphicx}
\usepackage{subcaption}
\usepackage{xstring}
\usepackage{eqparbox}

\makeatletter
\renewcommand{\p@subfigure}{\thefigure}
\makeatother

\setcitestyle{sort&compress,square,super,numbers,comma}
\usepackage{mathrsfs} 
\usepackage{mathtools}
\DeclareMathOperator*{\argmax}{arg\,max}
\DeclareMathOperator*{\argmin}{arg\,min}

\newtheorem{claim}{Claim}
\usepackage{thmtools, thm-restate}
\usepackage{bbm}
\usepackage{array,colortbl,multirow,multicol,booktabs}
\usepackage{longtable}
\DeclarePairedDelimiter{\ceil}{\lceil}{\rceil}

\linenumbers

\makeatletter
\let\Ginclude@graphics\@org@Ginclude@graphics 
\makeatother

\jmlrvolume{260}
\jmlryear{2024}
\jmlrworkshop{ACML 2024}
\jmlrpages{-}
\renewcommand{\thepage}{}
\nolinenumbers

\title[EVaDE : Event-Based Variational Thompson Sampling]{EVaDE : Event-Based Variational Thompson Sampling for Model-Based Reinforcement Learning}

\author{\Name{Siddharth Aravindan} \Email{siddharth.aravindan@gmail.com}\\
 \Name{Dixant Mittal} \Email{dixant@comp.nus.edu.sg}\\
  \Name{Wee Sun Lee} \Email{leews@comp.nus.edu.sg}\\
  \addr Department of Computer Science, National University of Singapore} 
 
\editors{Vu Nguyen and Hsuan-Tien Lin}

\begin{document}

\maketitle


\input{abstract}

\begin{keywords}
Exploration; Thompson Sampling; Model-Based Reinforcement Learning
\end{keywords}

\input{introduction}

\input{related_work}
\input{EVADE}

\input{Experiments}

\input{conclusion}


\clearpage
\appendix
\input{algorithm}
\input{theorem_proof}
\input{var_dropout}
\input{network_architectures}
\input{experimental_details}

{\tiny
\bibliography{acml24_submission_template}
}










\end{document}

%% file: abstract.tex
\begin{abstract}
Posterior Sampling for Reinforcement Learning (PSRL) is a well-known algorithm that augments model-based reinforcement learning (MBRL) algorithms with Thompson sampling. PSRL maintains posterior distributions of the environment transition dynamics and the reward function, which are intractable for tasks with high-dimensional state and action spaces. Recent works show that dropout, used in conjunction with neural networks, induces variational distributions that can approximate these posteriors. In this paper, we propose Event-based Variational Distributions for Exploration (EVaDE), which are variational distributions that are useful for MBRL, especially when the underlying domain is object-based. We leverage the general domain knowledge of object-based domains to design three types of event-based convolutional layers to direct exploration. These layers rely on Gaussian dropouts and are inserted between the layers of the deep neural network model to help facilitate variational Thompson sampling. We empirically show the effectiveness of EVaDE-equipped Simulated Policy Learning (EVaDE-SimPLe) on the 100K Atari game suite.
\end{abstract}

%% file: introduction.tex
\section{Introduction}
Model-Based Reinforcement Learning (MBRL) has recently gained popularity for tasks that allow for a very limited number of interactions with the environment \citep{Kaiser2020Model}. These algorithms use a model of the environment, that is learnt in addition to the policy, to improve sample efficiency in several ways; these include generating artificial training examples \citep{Kaiser2020Model,sutton1991dyna}, assisting with planning \citep{DBLP:journals/corr/abs-1708-02596,coulom2006efficient,williams2015model,curi2020efficient} and guiding policy search \citep{levine2013guided,chebotar2017path}. Additionally, it is easier to incorporate inductive biases derived from the domain knowledge of the task for learning the model, as the biases can be directly built into the transition and reward functions.

In this paper, we demonstrate how domain knowledge can be utilised for designing exploration strategies in MBRL. While model-free agents explore the space of policies and value functions, MBRL agents explore the space of transition dynamics and reward functions.

\begin{figure}
\centering
\hspace*{\fill}
{ \subfigure {
\includegraphics[width=0.15\linewidth, height=2.5cm]{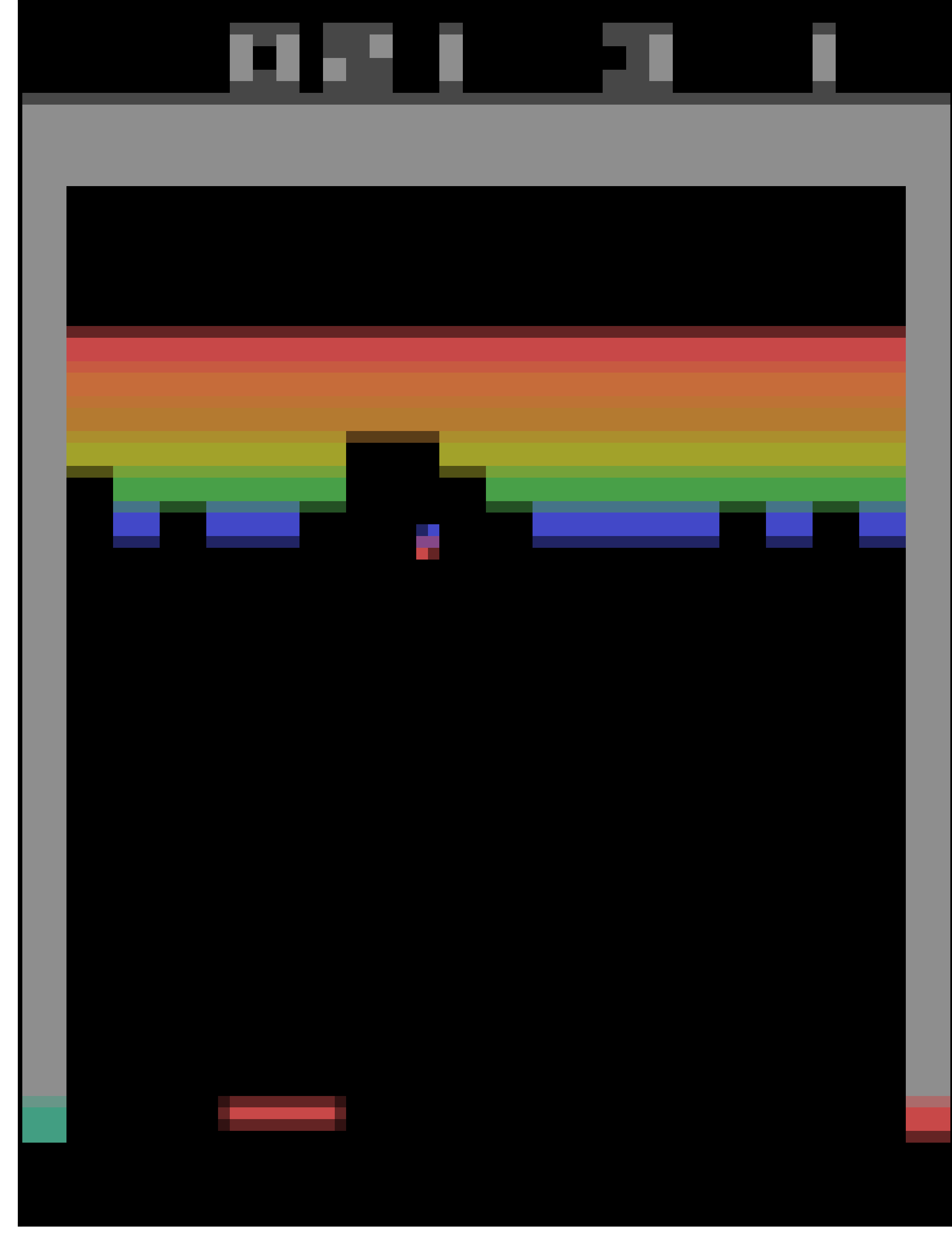}
    }
\hfill
 \subfigure{
\includegraphics[width=0.15\linewidth, height=2.5cm]{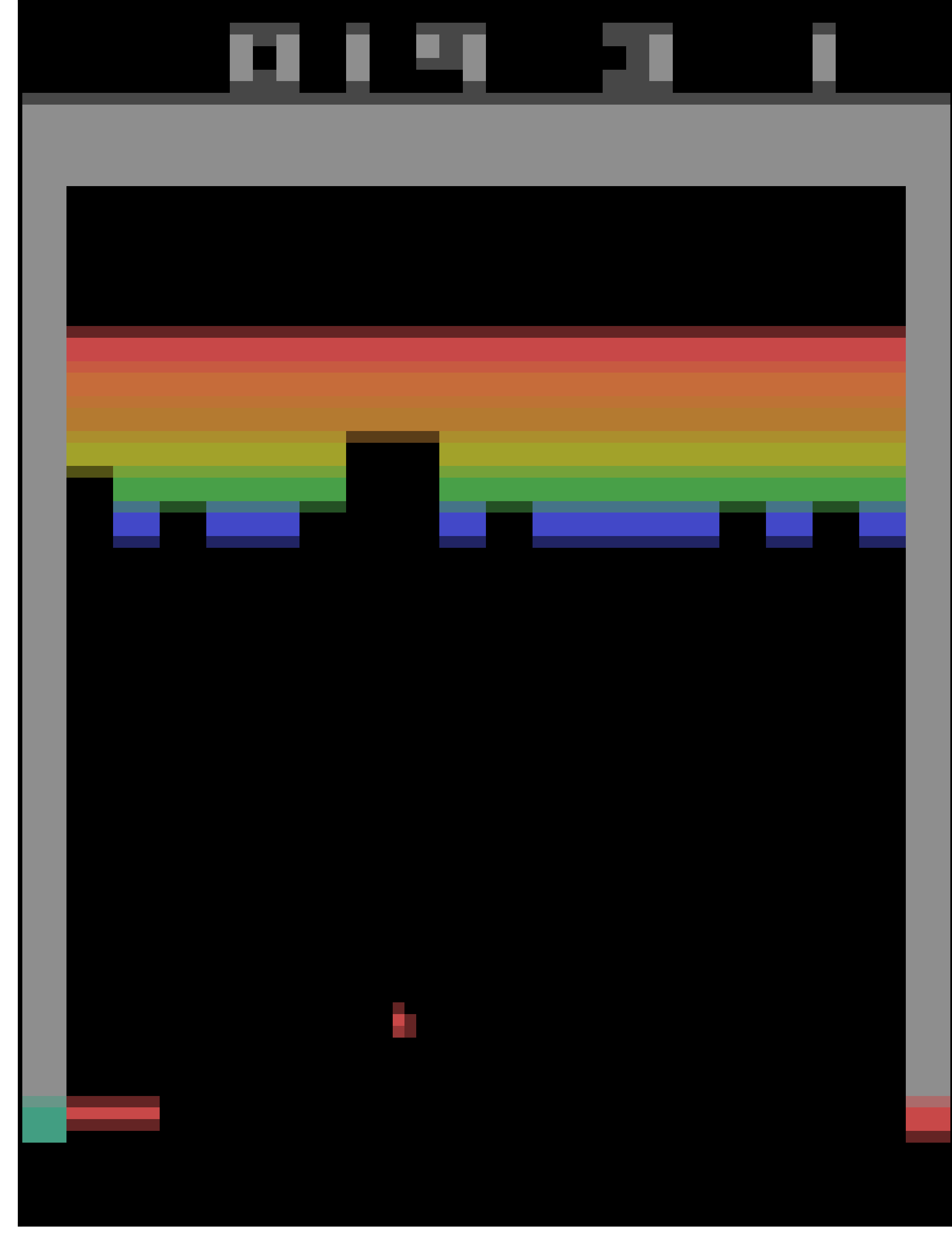}
    }
\hspace*{\fill}
}
\caption{Rewards in Breakout, a popular Atari game. (a) shows an interaction between the ball and a brick which gives the agent a positive reward. (b) shows a state, where the paddle is unable to prevent the ball from going out of bounds. The lack of this interaction between the agent and the ball in this situation results in a penalty for the agent.}
\label{fig:breakout_example}
\end{figure}

One method for exploring the space of transition dynamics and reward functions is Posterior Sampling for Reinforcement Learning (PSRL) \citep{strens2000bayesian,osband2017posterior}, which uses the Thompson sampling \citep{thompson1933likelihood} method of sampling the posterior of the model to explore other plausible models. Maintaining the posterior is generally intractable and in practice, variational distributions are often used as an approximation to the posterior \citep{aravindan2021state,pmlr-v119-wang20ab,zhang2019scalable}. 

Traditionally, variational distributions are designed with two considerations in mind: inference and/or sampling should be efficient with the variational distribution, and the variational distribution should approximate the true posterior as accurately as possible. However, as the variational distribution may not fully represent the posterior, different approximations may be suitable for different purposes. In this paper, we propose to design the variational distribution to generate trajectories through parts of the state space that may potentially give high returns, for the purpose of exploration.

In MBRL, trajectories are generated in the state space by running policies that are optimised against the learned model. One way to generate useful exploratory trajectories is by perturbing the reward function in the model, so that a different part of the state space appears to contain high rewards, leading the policy to direct trajectories towards those states. Another method is to perturb the reward function, so that the parts of the state space traversed by the current policy appear sub-optimal, causing the policy to seek new trajectories.

We focus on problems where the underlying domain is object-based, meaning that the reward functions heavily depend on the locations of individual objects and their interactions, which we refer to as events. An example of such an object-based task is the popular Atari game Breakout (as shown in Figure \ref{fig:breakout_example}). In this game, the agent receives rewards when the ball hits a brick and avoids losing a life by keeping the ball within bounds using the paddle, both of which are interactions between two objects. The rewards in this game are determined by the interactions between the ball and the bricks or the paddle.

For such domains, we introduce Event-based Variational Distributions for Exploration (EVaDE), a set of variational distributions that can help generate useful exploratory trajectories for deep convolutional neural network models. EVaDE comprises of three Gaussian dropout-based convolutional layers~\citep{JMLR:v15:srivastava14a}: the noisy event interaction layer, the noisy event weighting layer, and the noisy event translation layer. The noisy event interaction layer is designed to provide perturbations to the reward function in states where multiple objects appear at the same location, randomly perturbing the value of interactions between objects. The noisy event weighting layer perturbs the output of a convolutional layer at a single location, assuming that the output of the convolutional filters captures events; this would upweight and downweight the reward associated with these events randomly. The noisy event translation layer perturbs trajectories that go through "narrow passages"; small translations can randomly affect the returns from such trajectories, causing the policy to explore different trajectories.
 
These EVaDE layers can be used as standard convolutional layers and inserted between the layers of the environment network models. When included in deep convolutional networks, the noisy event interaction layers, the noisy event weighting layers, and the noisy event translation layers generate perturbations on possible object interactions, the importance of different events, and the positional importance of objects/events, respectively, through the dropout mechanism. This mechanism induces variational distributions over the model parameters \citep{JMLR:v15:srivastava14a,gal2016dropout}.

An interesting aspect of designing for exploration is that the variational distributions can be useful, even if they do not approximate the posterior well, as long as they assist in perturbing the policy out of local optima. After perturbing the policy, incorrect parts of the model will either be corrected by data or left unchanged if they are irrelevant to optimal behaviour.

Finally, we approximate PSRL by incorporating EVaDE layers into the reward models of Simulated Policy Learning (SimPLe)~\citep{Kaiser2020Model}. We conduct experiments to compare EVaDE-equipped SimPLe (EVaDE-SimPLe) with various popular baselines on the 100K Atari test suite. In the conducted experiments, all agents operate in the low data regime, where the number of interactions with the real environment is limited to 100K. EVaDE-SimPLe agents achieve a mean human-normalised score (HNS) of 0.682 in these games, which is 79\% higher than the mean score of 0.381 achieved by a recent low data regime method, CURL~\citep{laskin2020curl}, and 30\% higher than the mean score of 0.525 achieved by vanilla SimPLe agents.  


    

%% file: related_work.tex
\section{Background and Related Work}

Posterior sampling approaches like Thompson Sampling \citep{thompson1933likelihood} have been one of the more popular methods used to balance the exploration exploitation trade-off. Exact implementations of these algorithms have been shown to achieve near optimal regret bounds \citep{NIPS2017_3621f145,jaksch2010near}. These approaches, however, work by maintaining a posterior distribution over all possible environment models and/or action-value functions. This is generally intractable in practice. Approaches  that work by maintaining an approximated posterior distribution \citep{osband2016generalization,azizzadenesheli2018efficient}, or approaches that use bootstrap re-sampling to procure samples, \citep{osband2016deep,osband2015bootstrapped}  have achieved success in recent times. 

Variational inference procures samples from distributions that can be represented efficiently while also being easy to sample. These variational distributions are updated with observed data to approximate the true posterior as accurately as possible. Computationally cost effective methods such as dropouts have been known to induce variational distributions over the model parameters \citep{JMLR:v15:srivastava14a,gal2016dropout}. Consequently, variational inference approaches that approximate the posterior distributions required by Thompson sampling have gained popularity \citep{aravindan2021state,pmlr-v119-wang20ab,urteaga2018variational,xie2018nadpex}. 

Model-based reinforcement learning improves sample complexity at the computational cost of maintaining and performing posterior updates to the learnt environment models.  Neural networks have been successful in modelling relatively complex and diverse tasks such as Atari games \citep{oh2015action,ha2018recurrent}. Over the past few years, variational inference has been used to represent environment models, with the intention to capture environment stochasticity \citep{hafner2019learning,babaeizadeh2018stochastic,gregor2018temporal}. 

SimPLe \citep{Kaiser2020Model} is one of the first algorithms to use MBRL to train agents to play video games from images. It is also perhaps the closest to EVaDE, as it not only employs an iterative algorithm to train its agent, but also uses an additional convolutional network assisted by an autoregressive LSTM based RNN to approximate the posterior of the hidden variables in the stochastic model. Thus, similar to existing methods \citep{hafner2019learning,babaeizadeh2018stochastic,gregor2018temporal}, these variational distributions are used for the purpose of handling environment stochasticity rather than improving exploration. 
To the contrary, EVaDE-SimPLe is an approximation to PSRL, that uses a Gaussian dropout induced variational distribution over deterministic reward functions solely for the purpose of exploration. Unlike SimPLe, which uses the stochastic model to generate trajectories to train its agent, EVaDE-SimPLe agents optimize for a deterministic reward model sampled from the variational distribution and a learnt transition model. Moreover, with EVaDE, these variational distributions are carefully designed so as to explore different object interactions, importance of events and positional importance of objects/events, that we hypothesize are beneficial for learning good policies in object-based tasks.

The current state of the art scores in the Atari 100K benchmark is achieved by EfficientZero \citep{ye2021mastering}, which was developed concurrently with our work. Its success is a consequence of combining several improvements proposed previously in addition to integrating tree search with learning to improve the policy executed by the agent.  We believe that the benefits of using the variational designs induced by the EVaDE layers proposed in this paper are complementary to such  search based methods, as these layers could be used in their reward models to guide the policy search  by generating useful exploratory trajectories, especially in object-based domains.  


%% file: EVADE.tex
\section{Event Based Variational Distributions}
\label{sec:evade}

\begin{figure*}[t!]
\centering
{ 
    \subfigure{
        \includegraphics[width=0.35\linewidth]{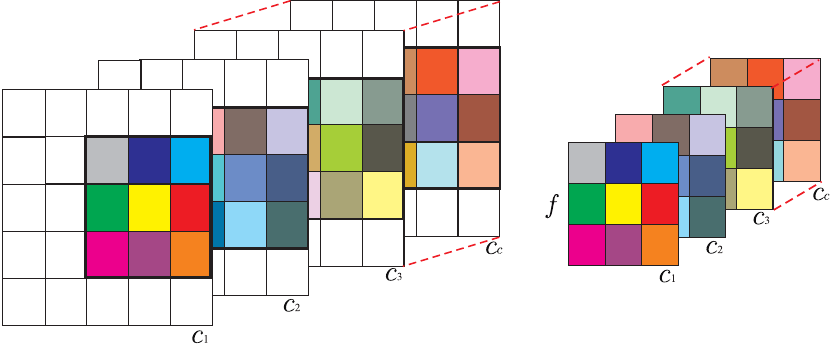}
        \label{fig:interaction_filter}
        }
    \subfigure {
        \includegraphics[width=0.2\linewidth]{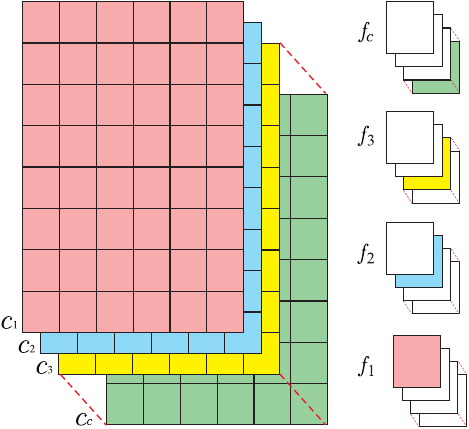}
        \label{fig:weighting_filter}
    }
    \subfigure[] {
        \includegraphics[width=0.35\linewidth]{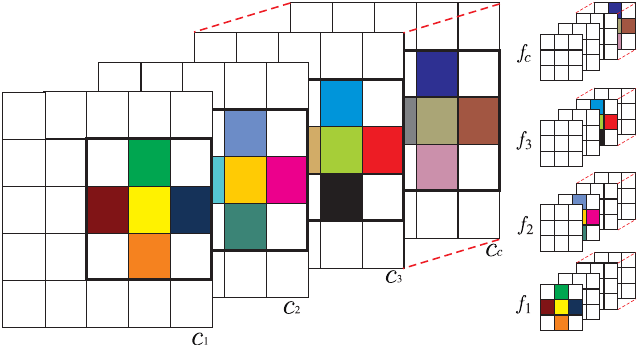}
        \label{fig:translation_filter}
    }
}
\caption{(a) This image shows one noisy event interaction filter acting on an input with $c$ channels. Here $f$ is an $m \times m$ noisy convolutional filter, which acts upon input patches at the same location across different channels, noisily altering the value of events captured at those locations. (b) This image shows how the filters of the noisy event weighting layer weight the input channels. The filters $f_1, f_2, f_3$ and $f_c$ randomly upweight and downweight the events captured by the channels $c_1, c_2, c_3$ and $c_c$ respectively. The white entries of the filter are entries that are set to zero, while the rest are trainable noisy model parameters. (c) The noisy event translation filter. The filters $f_1, f_2, f_3$ and $f_c$ noisily translate events/objects captured by the channels $c_1, c_2, c_3$ and $c_c$ respectively. The white entries of the filter are entries that are set to zero, while the rest are trainable noisy model parameters. Gaussian multiplicative dropout is applied to all the non-zero parameters of all EVaDE filters.}
\label{fig:all_states_low_risk}
\end{figure*}

Event-based Variational Distributions for Exploration (EVaDE) consist of a set of variational distribution designs, each induced by a noisy convolutional layer. These convolutional layers can be inserted after any intermediate hidden layer in deep convolutional neural networks to help us construct approximate posteriors over the model parameters to produce samples from relevant parts of the model space. EVaDE convolutional layers use Gaussian multiplicative dropout to draw samples from the variational approximation of the posterior. Posterior sampling is done by multiplying each parameter, $\theta_{env}^i$, of these EVaDE layers by a perturbation drawn from a Gaussian distribution, $\mathcal{N}(1,(\sigma_{env}^{i})^{2})$. These perturbations are sampled by leveraging the reparameterization trick \citep{kingma2015variational,salimans2017evolution, plappert2018parameter, fortunato2018noisy} using a noise sample from the standard Normal distribution, $\mathcal{N}(0,1)$, as shown in Equation \ref{eq:reparameterization}. The variance corresponding to each parameter, $(\sigma_{env}^{i})^{2}$, is trained jointly with the model parameters $\theta_{env}$.    
\begin{equation}
    \Tilde{\theta}^i_{env} \leftarrow \theta^i_{env}(1 + \sigma^i_{env} \epsilon^i ) ; \;\; \epsilon^i \sim \mathcal{N}(0,1)
    \label{eq:reparameterization}
\end{equation}

When the number of agent-environment interactions is limited, the exploration strategy employed by the agent is critical. In object-based domains, rewards and penalties are often sparse and occur when objects interact. Hence, the agent needs to experience most of the events in order to learn a good environment model. Generating trajectories that contain events is hence a reasonable exploration strategy.
Additionally, a very common issue with training using a very few number of interactions is that the agent may often get stuck in a local optimum, prioritising an event, which is relatively important, but may not lead to an optimal solution. Generating potentially high return alternate trajectories that do not include that event is another reasonable exploration strategy.

With these exploration strategies in mind, we introduce three EVaDE layers, namely the noisy event interaction layer, the noisy event weighting layer and the noisy event translation layer. All the three layers are constructed with the hypothesis that the channels of the outputs of intermediate layers of deep convolutional neural networks either capture object positions, or events (interaction of multiple objects detected by multi-layer composition of the network).

\subsection{Noisy Event Interaction Layer}
The  noisy event interaction layer is designed with the motivation of increasing the variety of events experienced by the agent. This layer consists of  noisy convolutional filters, each having a dimension of $m \times m \times c$, where $c$  is the number of input channels to the layer. Every filter parameter is multiplied by a Gaussian perturbation as shown in Equation \ref{eq:reparameterization}. The filter dimension, $m$, is a hyperparameter that can be set so as to capture objects within a small $m \times m$ patch of an input channel. Assuming that the input channels capture the positions of different objects, a filter that combines the $c$ input channels locally captures the local object interaction within the $m\times m$ patch. By perturbing the filter, different combinations of interactions can be captured; if the filter is used as part of the reward function, it will correspondingly reward different interactions depending on the perturbation. The policy optimized for different perturbed reward functions is likely to generate trajectories that contain different events. Note that convolutional filters are equivariant, so the same filter will detect the event anywhere in the image and can result in trajectories that include the event at different positions in the image.

We describe the filter in more detail. Every output pixel of the filter, $y_{i,j}^{k}$, representing $(i,j)^{th}$ pixel of the $k^{th}$ output channel, can be computed as shown in Equation \ref{eq:ev_int}.  Here $x$ is the input to the layer with $c$ input channels, $P_{x_{i,j}^{l}}$ is the $m\times m$ patch (represented as a matrix) centred around  $x_{i,j}^{l}$, the $(i,j)^{th}$ pixel of the $l^{th}$ input  channel, $\Tilde{\theta}^{l}_{k}$ is the $l^{th}$ channel of the $k^{th}$ noisy convolutional filter, $\odot$ the Hadamard product operator, and $\mathbbm{1}_m$ is an $m$ dimensional column vector whose every entry is $1$.  
\begin{equation}
    y_{i,j}^k  = \sum\limits_{l=0}^c \mathbbm{1}_{m}^{T} \left(\Tilde{\theta}^{l}_{k} \odot P_{x_{i,j}^{l}}\right)\mathbbm{1}_{m}
    \label{eq:ev_int}
\end{equation}
Figure \ref{fig:interaction_filter} shows how this filter is applied over the input channels.


\subsection{Noisy Event Weighting Layer}
\label{sec:ev_we}
Overemphasis on certain events is possibly one of the main causes due to which agents converge to sub-optimal policies in object based tasks. Hence, it would be useful to easily be able to increase as well as decrease the importance of an event. For this layer, we assume that each input channel is already detecting an event and design a variational distribution over model parameters that directly up-weights or down-weights the events captured by different input channels.

This layer can be implemented with the help of $c$ $1 \times 1$ noisy convolutional filters (each having a dimension of $1 \times1 \times c$ as shown in Figure \ref{fig:weighting_filter}), where $c$ is the number of input channels.  We denote the $l^{th}$ element of the $k^{th}$ filter in the layer as $\theta^l_k$. To implement per channel noisy weighting, we set every $\theta^k_k$ as a trainable model parameter, which has a Gaussian dropout variance parameter associated with it to facilitate noisy weighting as shown in Equation \ref{eq:reparameterization}.
All other weights, i.e., $\theta^l_k$ when $l \neq k$ are set to 0. 
Thus each noisy event weighting layer has $c$ trainable model parameters and $c$ trainable Gaussian dropout parameters. A pictorial representation of how this layer acts on its input is presented in Figure \ref{fig:weighting_filter}.

Every output $y_{i,j}^k$, corresponding to the $(i,j)^{th}$ pixel of the $k^{th}$ output channel, can be computed using Equation \ref{eq:ev_we}, where $\Tilde{\theta}^{k}_{k}$ is the noisy scaling factor for the $k^{th}$ input channel.
\begin{equation}
    y_{i,j}^k  =  \Tilde{\theta}^{k}_{k}x_{i,j}^{k}
    \label{eq:ev_we}
\end{equation}
We expect that inducing such a variational distribution that up-weights or downweights events randomly helps the agents learn from different events that are randomly emphasised by different model samples drawn from the distribution. This may eventually help them in escaping local optima caused by overemphasis of certain events.

\subsection{Noisy Event Translation Layer}
\label{sec:ev_tr}
In object based domains, an agent often has to perform a specific sequence of actions to successfully gain some rewards and may be penalized heavily for deviation from the sequence. We refer to the specific sequence of actions as a "narrow passage". A small translation of the positions of the environment or other objects will often cause the agent to be unsuccessful. When random translations of obstacles, events or boundaries are performed within the reward function, the optimized policy may select a different trajectory, possibly allowing it to escape from a locally optimal trajectory. We thus design the noisy event  translation layer to induce a variational distribution over such model posteriors that can sample a variety of translations of relevant objects.

The noisy soft-translation on an input with $c$ channels, is performed with the help of $c$ convolutional  filters, each having a dimension of $m \times m \times c$. These filters compute a noisy weighted sum of the corresponding input pixel and the pixels near it to effect a \textit{noisy} translation of the channel. Similar to the noisy event weighting layer, each filter of the noisy event translation layer acts on one input channel. To achieve this, every parameter except the parameters of the $k^{th}$ channel of the $k^{th}$ filter, $\theta^k_k$ (which has a dimension of $m \times m$), and their corresponding dropout variances, is set to 0, for all $k$. 
Moreover in the channel $\theta_k^k$,  only the middle column and row contain trainable parameters. Figure \ref{fig:translation_filter} shows a detailed pictorial representation of this structure of the filters.   A random translation of up to $n$ pixels of the input can be achieved by using a $(2n+1)\times (2n+1)$ noisy event translation layer.  

Equation \ref{eq:ev_tr} shows how $y_{i,j}^k$, the $(i,j)^{th}$ output pixel of the $k^{th}$ channel, is computed. Here, $P_{x_{i,j}^{k}}$ is a $m \times m$ patch centred at $(i,j)^{th}$ pixel of the $k^{th}$ input channel,  $\Tilde{\theta}^{k}_{k}$ is the $k^{th}$ channel of the $k^{th}$ noisy convolutional filter, $\odot$ the Hadamard product operator, and $\mathbbm{1}_m$ is an $m$ dimensional column vector where all the entries are $1$.
\begin{equation}
        y_{i,j}^k  = \mathbbm{1}_{m}^{T} \left( \Tilde{\theta}^{k}_{k} \odot P_{x_{i,j}^{k}}\right)\mathbbm{1}_{m}
    \label{eq:ev_tr}
\end{equation}

\subsection{Representational Capabilities of EVaDE networks}
Ideally, adding EVaDE layers for exploration should not hinder the network  to be unable to represent the true model, even if they don't accurately approximate the posterior. Theorem \ref{th} below states that this is indeed the case.

\begin{theorem}
\label{th}
Let $\mathbbm{n}$ be any neural network. For any convolutional layer $l$, let $m_i(l) \times n_i(l) \times c_i(l)$ and $m_o(l) \times n_o(l) \times c_o(l)$ denote the dimensions of its input and output respectively. Then, any function that can be represented by $\mathbbm{n}$ can also be represented by any network $\Tilde{\mathbbm{n}} \in \Tilde{\mathscr{N}}$, where $\Tilde{\mathscr{N}}$ is the set of all neural networks that can be constructed by adding any combination of EVaDE layers to  $\mathbbm{n}$, provided that, for every EVaDE layer $\Tilde{l}$ added, $\Tilde{l}$ uses a  stride of $1$, $m_i(\Tilde{l}) \leq m_o(\Tilde{l}), n_i(\Tilde{l}) \leq n_o(\Tilde{l})$ and $c_i(\Tilde{l}) \leq c_o(\Tilde{l})$.  
\end{theorem}

\proof {
The proof follows from the fact that every EVaDE layer $\Tilde{l_i}$ that is added is capable of representing the identity function. A detailed proof is presented in the supplementary material. 
}\\

If the added EVaDE layers induce distributions that poorly approximate the posterior, performance can indeed be poorer. But with enough data, the correct model should still be learnable since it is representable, as long as the optimization does not get trapped in a poor local optimum.
\begin{figure*}
\centering
\includegraphics[width=0.98\textwidth, height=5cm]{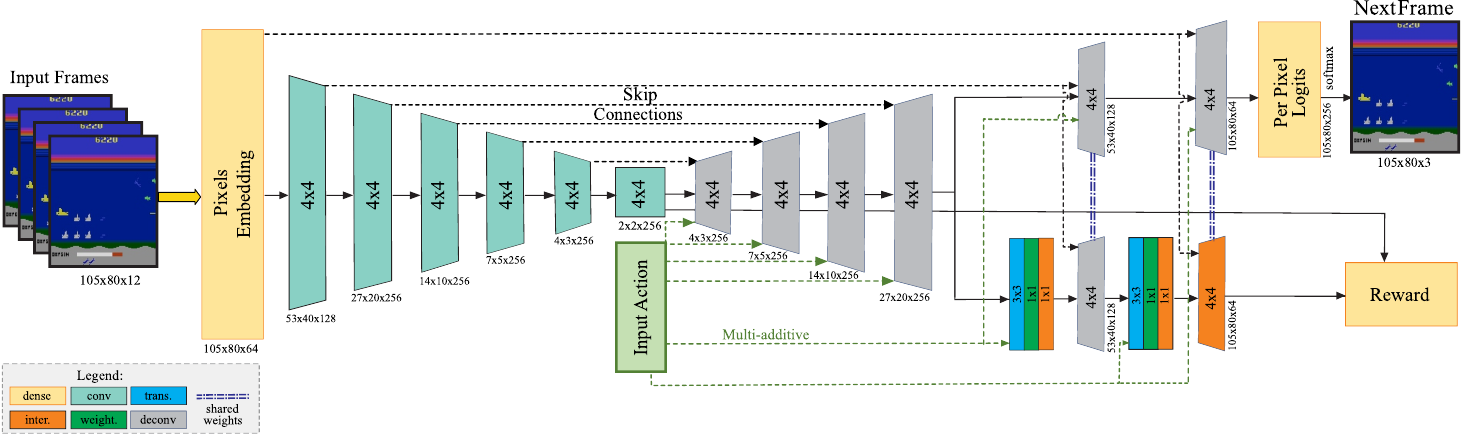}
\caption{The network architecture of the environment model used to train EVaDE-SimPLe.}
\label{fig:evade_arch}
\end{figure*} 
\subsection{Approximate PSRL with EVaDE equipped Simulated Policy Learning}
Simulated Policy Learning (SimPLe) \citep{Kaiser2020Model} is an iterative model based reinforcement learning algorithm, wherein the environment model learnt is used to generate artificial episodes to train the agent policy. In every iteration, the SimPLe agent first interacts with the real environment using its current policy. After being trained on the set of all collected interactions, the models of the transition and reward functions are then used as a substitute to the real environment to train the policy of the agent to be followed by it in its next interactions with the real environment. PSRL \citep{strens2000bayesian,osband2017posterior}, which augments MBRL with Thompson sampling, has a very similar iterative structure as that of SimPLe. However, instead of maintaining a single environment model, PSRL maintains a posterior distribution over all possible environment models given the interactions experienced by the agent with the real environment. The agent then optimizes a policy for an environment model sampled from this posterior distribution. This policy is used in its real environment interactions of the next iteration. EVaDE equipped SimPLE approximates PSRL, by maintaining an approximate posterior distribution of the reward function with the help of the variational distributions induced by the three EVaDE layers.


Being an approximation of PSRL, an EVaDE-SimPLe agent has the same iterative training structure where it acts in the real environment using its latest policy to collect training interactions, learns a transition model and an approximate posterior over the reward model by jointly optimizing the environment model parameters, $\theta_{env}$, and the Gaussian dropout parameters of the reward model, $\sigma_{rew}$, with the help of supervised learning. It then optimizes its policy with respect to an environment characterized by the learnt transition function and a reward model sample that is procured from the posterior with the help of Gaussian dropout as shown in Equation \ref{eq:reparameterization}. This policy is then used by the agent to procure more training interactions in the next iteration.










%% file: Experiments.tex
\section{Experiments}
\label{sec:experiments}
We conduct our experiments on the 100K Atari game test suite first introduced by \citep{Kaiser2020Model}. This test suite consists of 26 Atari games where the 
 number of agent-environment interactions is limited to 100K.  Due to its diverse range of easy and hard exploration games \citep{bellemare2016unifying}, this test-suite  has become a popular test-bed for evaluating reinforcement algorithms.

\subsection{Network Architecture}

In our experiments we use the network architecture of the deterministic world model introduced by \citep{Kaiser2020Model}  to train the environment models of the SimPLe agents, but do not augment it with the convolutional inference network and the autoregressive LSTM unit. 
Readers are referred to  \cite{Kaiser2020Model} for more details.

The architecture of the environment model used by EVaDE-SimPLe agents is shown in Figure \ref{fig:evade_arch}. This model is very similar to the one used by SimPLe agents, except that it has a combination of a $3 \times 3$ noisy event translation layer, a noisy event weighting layer and a $1\times 1$ noisy event interaction layer inserted before the fifth and sixth de-convolutional layers. The final de-convolutional layer acts as a noisy event interaction filter when computing the reward, while it acts as a normal de-convolutional layer while predicting the next observation. Sharing weights between layers allows us to achieve this. We insert EVaDE layers in a way that it  perturbs only the reward function and not the transition dynamics.  

 We reuse the network architecture of \citep{Kaiser2020Model} to train the policies in both the SimPLe and EVaDE-SimPLe agents using  Proximal Policy Optimization (PPO) \citep{DBLP:journals/corr/SchulmanWDRK17}. All the hyperparameters used for training the policy network and environment are the same as the ones used in \citep{Kaiser2020Model}.

\subsection{Experimental Details}
The training regimen that we use to train all the agents is the same and is structured similarly to the setup used by \citep{Kaiser2020Model}. The agents, initialized with a random policy and collect 6400 real environment interactions before starting the first training iteration.  In every subsequent iteration, every agent trains its environment model with its collection of real world interactions, refines its policy by interacting with the environment model, if it is a vanilla-SimPLe agent, or a transition model and a reward model sampled from the approximate posterior, if it is an EVaDE-SimPLe agent, and then collects more interactions with this refined policy. 

PSRL regret bounds scale linearly with the length of an episode experienced by the agent in every iteration \citep{osband2013more}. Shorter horizons, however, run the risk of the agent not experiencing anything relevant before episode termination. To balance these factors, we set the total number of iterations to 30, instead of 15. We allocate an equal number of environment interactions to each iteration, resulting in 3200 agent-environment interactions per iteration. The total number of interactions that each SimPLe and EVaDE-SimPLe agent procures (about 102K) is similar to SimPLe agents trained in \citep{Kaiser2020Model}, which allocates double the number of interactions per iteration, but trains for only 15 iterations. To disambiguate between the different SimPLe agents referred to in this paper, we refer to the SimPLe agents trained in our paper and \citep{Kaiser2020Model} as SimPLe(30) and SimPLe respectively from here on. 

\begin{table*}[t]
  \scriptsize
      \caption{ Comparison of the performances achieved by popular baselines and five independent training runs of EVaDE-SimPLe and SimPLe(30) agents with 100K agent-environment interactions in the 26 game Atari 100K test suite.}
    \label{tab:EVADE_sota}
\centering
  \begin{tabular}{c@{\hspace{0.9\tabcolsep}}c@{\hspace{0.9\tabcolsep}}c@{\hspace{0.9\tabcolsep}}c@{\hspace{0.9\tabcolsep}}c@{\hspace{0.9\tabcolsep}}c@{\hspace{0.9\tabcolsep}}c}
    \toprule
    Game & SimPLe&SimPLe(30)&CURL&OTRainbow&Eff. Rainbow&EVaDE-SimPLe\\
    \midrule
        Mean HNS & 0.443 & 0.525& 0.381 & 0.264  & 0.285 & \textbf{0.682} \\
        Median HNS & 0.144 & 0.151& 0.175 & 0.204  & 0.161 & \textbf{0.267} \\
        Vs EVaDE (W/L) &7W,19L&3W,23L &9W,17L&6W,20L&9W,17L&- \\
        Best Performing & 5& 2& 4& 1& 3& \textbf{11}\\  
    \bottomrule
  \end{tabular}

\end{table*}

We try to keep the training schedule of EVaDE-SimPLe and SimPLe(30) similar to the training schedule of the deterministic model in \citep{Kaiser2020Model} so as to keep the comparisons fair. We train  the environment model for 45K steps in the first iteration and 15K steps in all subsequent iterations.  In every iteration of simulated policy training, 16 parallel PPO agents collect $z*1000$ batches of 50 environment interactions each, where  $z=1$ in all iterations except iterations 8, 12, 23 and 27 where $z=2$ and in iteration 30, where  $z=3$. The policy is also trained when the agent interacts with the real environment. However, the effect of these interactions (numbering 102K) on the policy  is minuscule when compared to  the 28.8M transitions experienced by the agent when interacting with the learnt environment model. Additional experimental details as well as the anonymized code for our agents are provided in the supplementary, which is available at \url{https://tinyurl.com/3zb8nywx}.

\subsection{Results}

We report the mean and median Human Normalized Scores (HNS) achieved by SimPLe(30), EVaDE-SimPLe  and popular baselines SimPLe \citep{Kaiser2020Model}, CURL \citep{laskin2020curl}, OverTrained Rainbow \citep{DBLP:journals/corr/abs-2003-10181} and Data-Efficient Rainbow \citep{DBLP:conf/nips/HasseltHA19} in Table \ref{tab:EVADE_sota}. For each baseline, we report 
 the number of games in which it is the best performing, among all compared methods, as well as the number of games in which it scores more (or less) than EVaDE-SimPLe, which are counted as its wins (or losses).

EVaDE-SimPLe agents achieve the highest score in 11 of the 26 games in the test suite, outperforming every other method on at least 17 games. Moreover, the effectiveness of the noisy layers to improve exploration can be empirically verified as EVaDE-SimPLe manages to attain higher mean scores than SimPLe(30) in 23 of the 26 games, even though both methods follow the same training routine. EVaDE-SimPLe also scores a mean HNS of 0.682, which is 79\% higher than the score of 0.381 achieved by a popular baseline, CURL, and 30\% more than the mean HNS of 0.525 achieved by SimPLe(30). Additionally, EVaDE-SimPLe agents also surpass the human performances \citep{DBLP:journals/corr/abs-1905-12726} in 5 games, namely Asterix, Boxing, CrazyClimber, Freeway and Krull. 

We also compute the Inter-Quartile Means (IQM), \footnote{IQM is well regarded in the reinforcement learning community, advocated by \citep{agarwal2021deep}, which won the Outstanding Paper Award at NeurIPS 2021} a metric resilient to outlier games and runs, of SimPLe(30) and EVaDE-SimPLe agents.  EVaDE-SimPLe agents achieve an IQM  of 0.339, which is 68\% higher than the IQM of 0.202 achieved by Simple(30) agents. This affirms that the improvements obtained due to the addition of the EVaDE layers are robust to outlier games and runs. In the supplementary material, we provide the scores achieved by all five independent runs of SimPLe(30) and EVaDE-SimPLe agents, which were used to compute the IQM.

Furthermore, a paired t-test on the mean HNS achieved by EVaDE-SimPLe and SimPLe(30) agents on each of the 26 games yields a single-tailed p-value of $3 \times 10^{-3}$ confirming that the performance improvements over SimPLe(30) of EVaDE-SimPLe agents are statistically significant as an algorithm when applied to multiple games. 

\begin{table*}[t]

  \scriptsize

   \caption{Scores (mean $\pm$ 1 standard error) achieved by SimPLe agents when equipped with all three EVaDE filters individually and when equipped with all filters simultaneously. All scores are averaged over five independent training runs.}
   
    \label{tab:EVADE_ablation}
\centering
  \begin{tabular}{c@{\hspace{1.1\tabcolsep}}c@{\hspace{1.1\tabcolsep}}c@{\hspace{1.1\tabcolsep}}c@{\hspace{1.1\tabcolsep}}c@{\hspace{1.1\tabcolsep}}c}

    \toprule
    Game & SimPle (30)&Inter. Layer &  Weight. Layer & Trans. Layer & EVaDE-SimPLe   \\
    \midrule
    BankHeist&78.6$\pm$31.7&107.5$\pm$29.2 &168.4$\pm$19.9 & 180.7 $\pm$ 16.7  & \textbf{224.2 $\pm$ 35.4}\\
    BattleZone&4544$\pm$803& 6688 $\pm$ 1617 & 7525$\pm$2164 & 7750 $\pm$ 1355  & \textbf{11094 $\pm$ 572} \\
    Breakout&18.9$\pm$1.7& 19.8 $\pm$ 3.6 & 22.4 $\pm$ 5.5& 19.5 $\pm$1.5 & \textbf{24 $\pm$ 3.4} \\
    CrazyClimber&43458$\pm$7709&59546$\pm$3164&\textbf{64191$\pm$5196}&59006$\pm$3282&60716$\pm$4082\\
    DemonAttack&120.7$\pm$18.2&136.3 $\pm$24.4& 132 $\pm$ 14.7& 133.7 $\pm$ 26 & \textbf{141.8 $\pm$ 12.5} \\
    Frostbite&260.3$\pm$2.5& 254.6$\pm$5.5& 254.4$\pm$3.6 & 263.2$\pm$ 2.1 & \textbf{274.2 $\pm$ 11}\\
    JamesBond&\textbf{245.6$\pm$11.2}&202.2$\pm$65.2 & 182.5 $\pm$ 56.2 & 160.3 $\pm$ 68.4 &235.6 $\pm$ 50.2 \\
    Kangaroo&576$\pm$330&\textbf{2201$\pm$993} & 837.5 $\pm$ 345 & 1297$\pm$ 321& 1186.5$\pm$168 \\
    Krull&4532$\pm$883&3117$\pm$781 & 4818 $\pm$440 &5185 $\pm$ 991 &\textbf{5335$\pm$455} \\
    Qbert&2583$\pm$746&1953$\pm$ 674& 932 $\pm$148 & \textbf{3333 $\pm$ 575} &2764 $\pm$ 783 \\
    RoadRunner &2385$\pm$888& 7178$\pm$1227 & 4853 $\pm$ 1322 & 6070$\pm$1834 &\textbf{7799 $\pm$ 1296} \\
    Seaquest &321.6$\pm$52& \textbf{644.4 $\pm$ 91.6} & 608.5 $\pm$ 144.9&644.2 $\pm$56.9 & 617.5 $\pm$ 118.1 \\
\midrule
        \textbf{HNS}& 0.52 & 0.56 & 0.65  & 0.69 & \textbf{0.77} \\
      \textbf{IQM}& 0.22 & 0.29 & 0.26  & 0.29 & \textbf{0.4} \\
      \textbf{Vs SimPLe(30) (W/L)} & -  &  8W,4L&9W,3L &11W,1L& 11W,1L\\
    \bottomrule
  \end{tabular}
\end{table*}

\subsection{Ablation Studies}
We also conduct ablation studies by equipping SimPLe(30) with just one of the three EVaDE layers to ascertain whether all of them aid in exploration. We do this by just removing the other two layers from the EVaDE environment network model (see Figure \ref{fig:evade_arch}).  Note that reward models that do not equip the noisy event interaction filter,  also do not apply the Gaussian multiplicative dropout to the sixth de-convolutional layer.

We use a randomly selected suite of 12 Atari games in our ablation study. The games were chosen by arranging the 26 games of the suite in the alphabetical order, and then using the numpy \cite{harris2020array} random function to sample 12 numbers from 0 to 25 without replacement. The corresponding games were then picked. Coincidentally, the chosen test suite contains easy exploration games such as Kangaroo, RoadRunner and Seaquest as well as BankHeist, Frostbite and Qbert, which are hard exploration games \citep{bellemare2016unifying}.

The mean scores, HNS and IQM achieved when SimPLe(30) is equipped with only one type of noisy convolutional layer and those of SimPLe(30) and EVaDE-SimPLe are presented in Table \ref{tab:EVADE_ablation}. 

\begin{figure*}[t!]
{ \scriptsize
\centering
\hspace*{\fill}
\captionsetup[subfigure]{labelformat=empty}
\subfigure{
\includegraphics[width=0.153\textwidth]{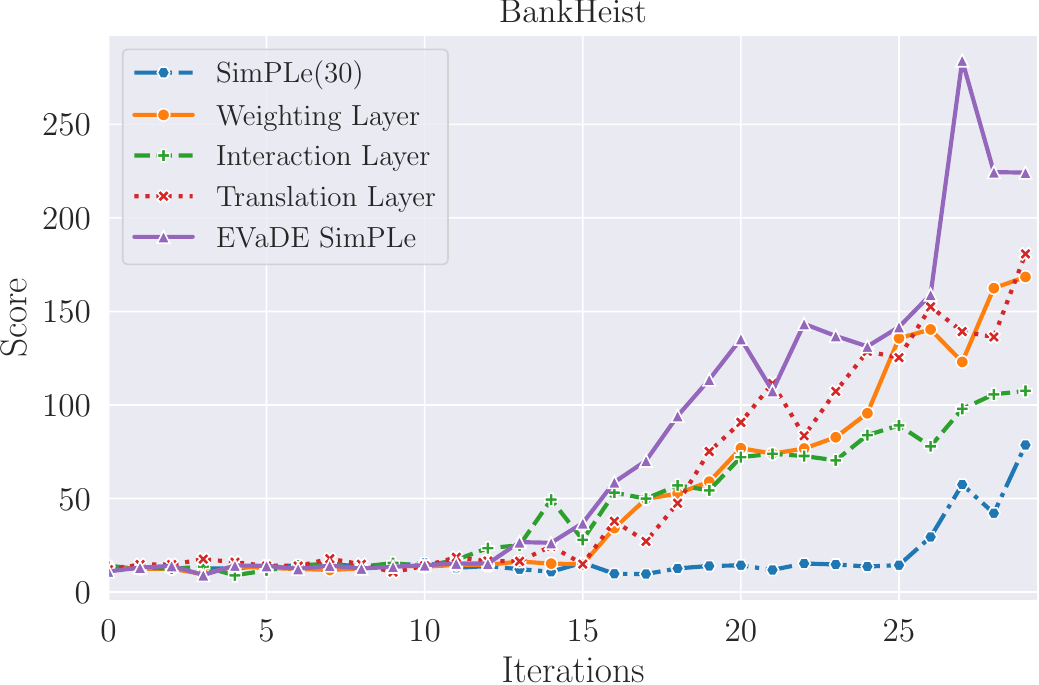}
}
\hfill
\subfigure{
\includegraphics[width=0.153\textwidth]{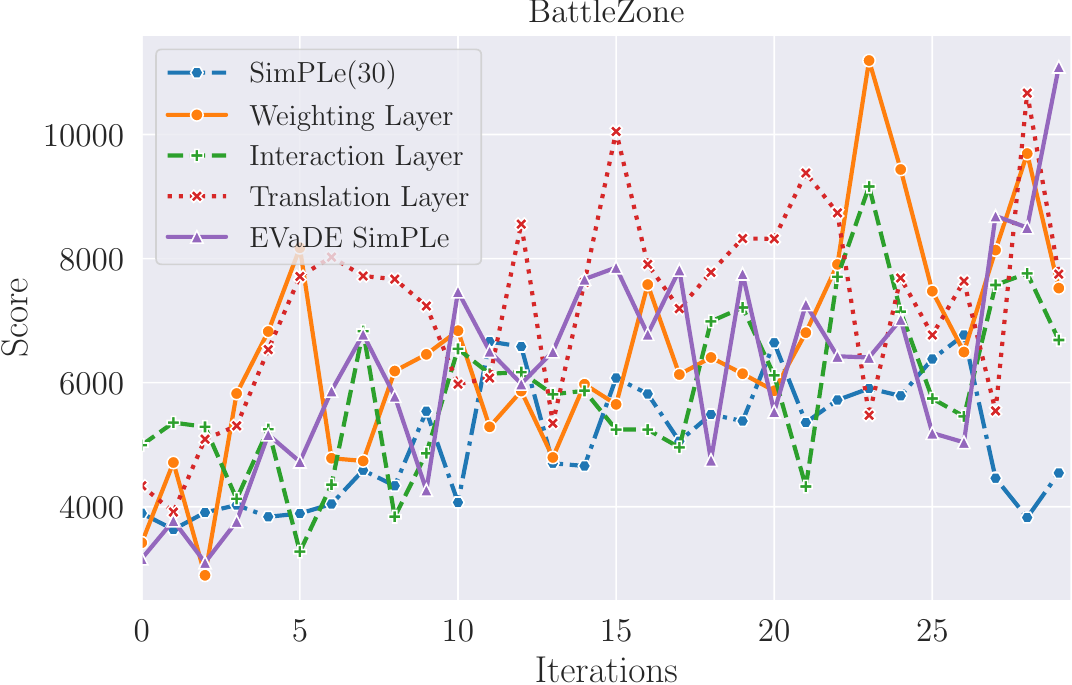}
}
\hfill
\subfigure{
\includegraphics[width=0.153\textwidth]{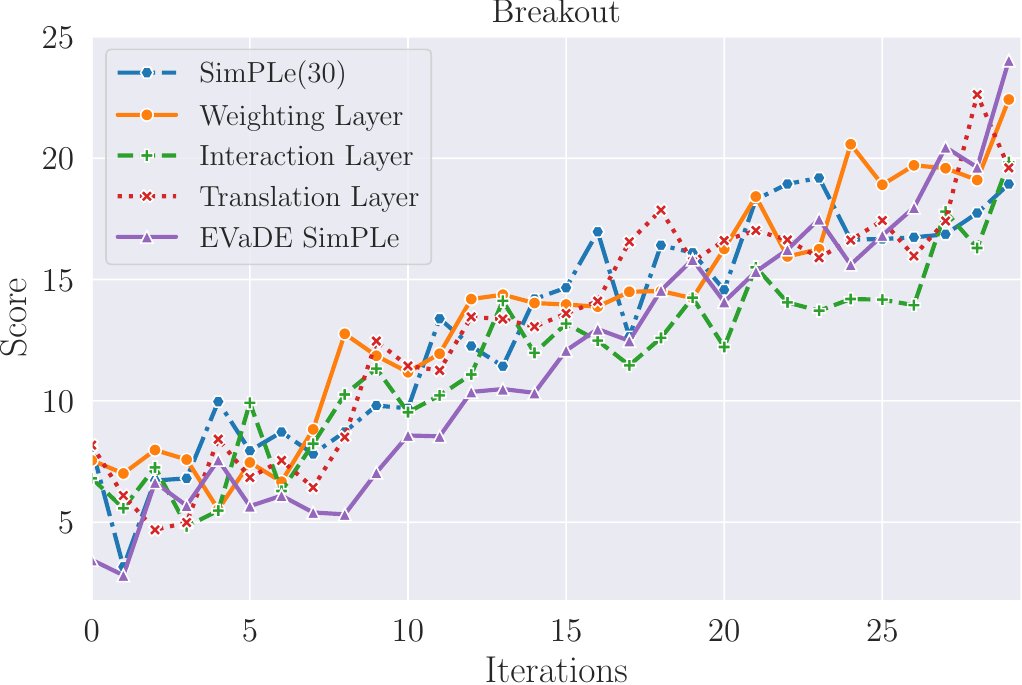}
}
\hfill
\subfigure{
\includegraphics[width=0.153\textwidth]{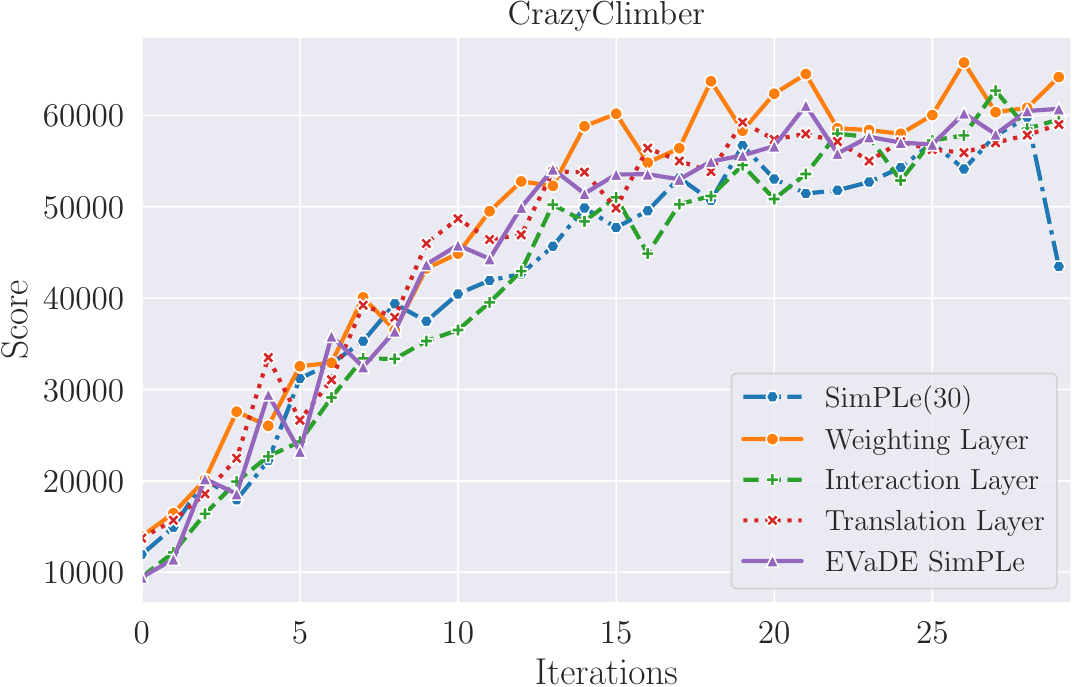}
}
\hfill
\subfigure{
\includegraphics[width=0.153\textwidth]{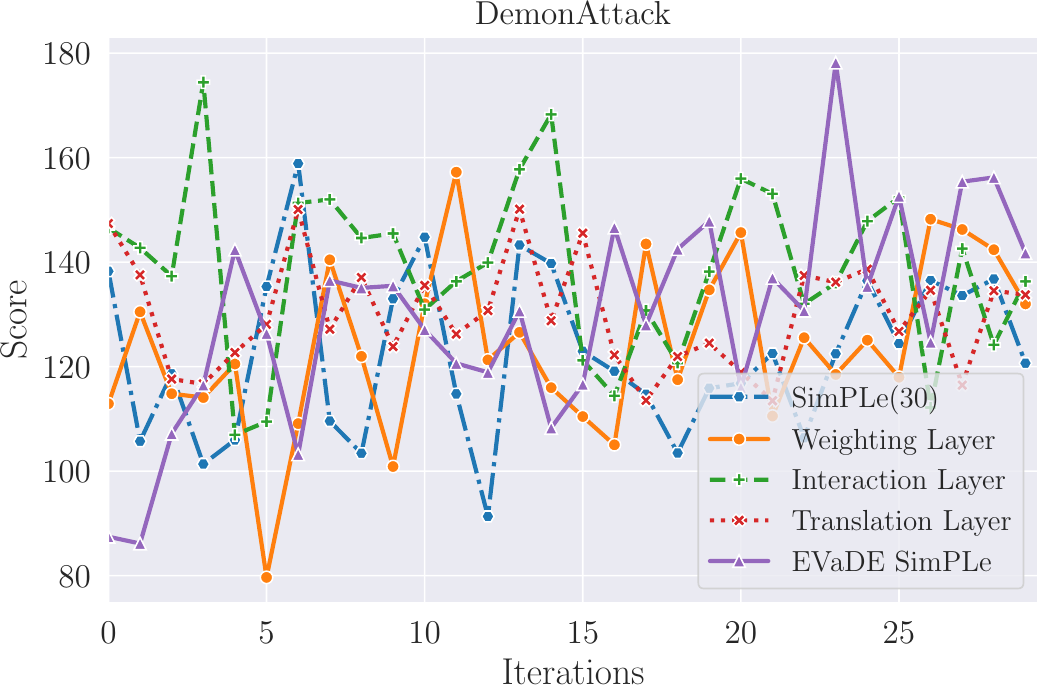}
}
\hfill
\subfigure{
\includegraphics[width=0.153\textwidth]{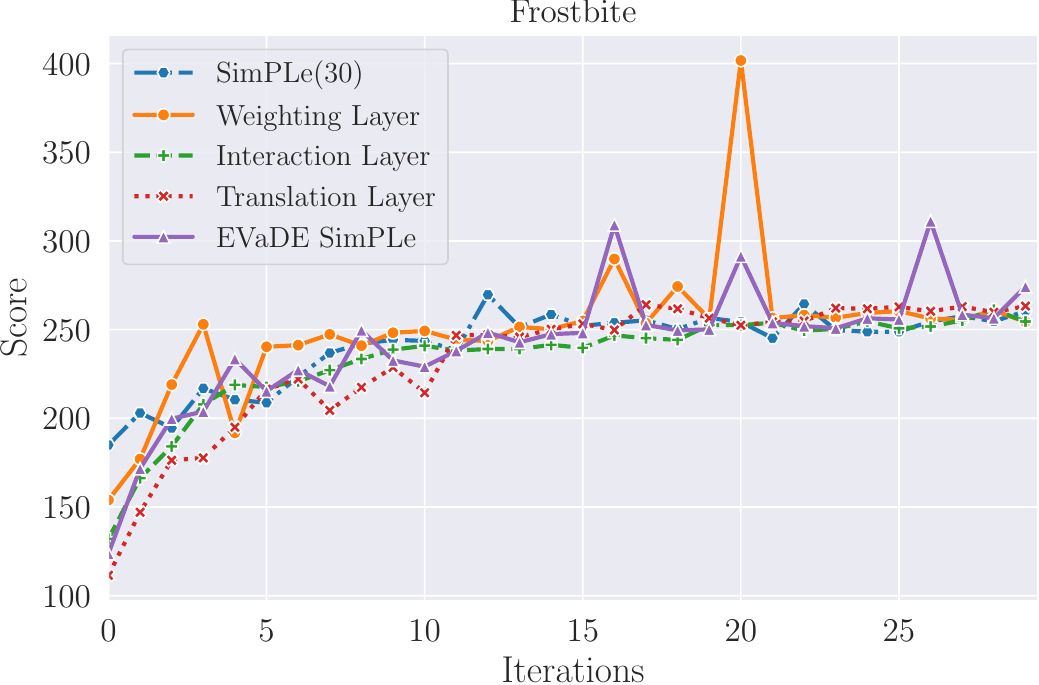}
}
\hspace*{\fill}
\newline
\subfigure{
\includegraphics[width=0.153\textwidth]{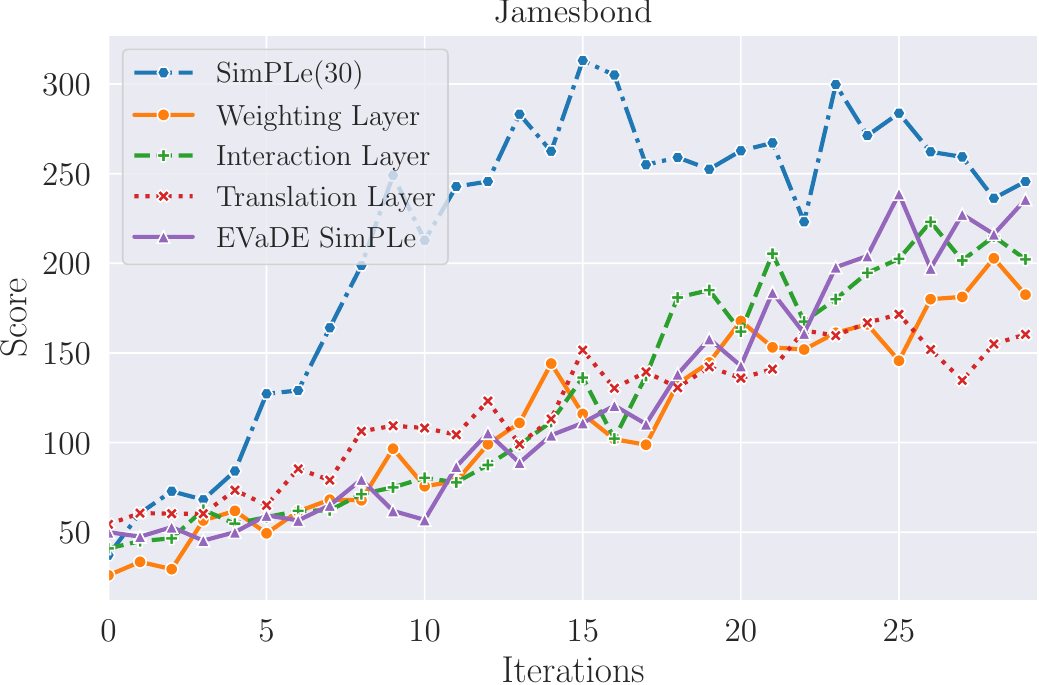}
}
\subfigure{
\includegraphics[width=0.153\textwidth]{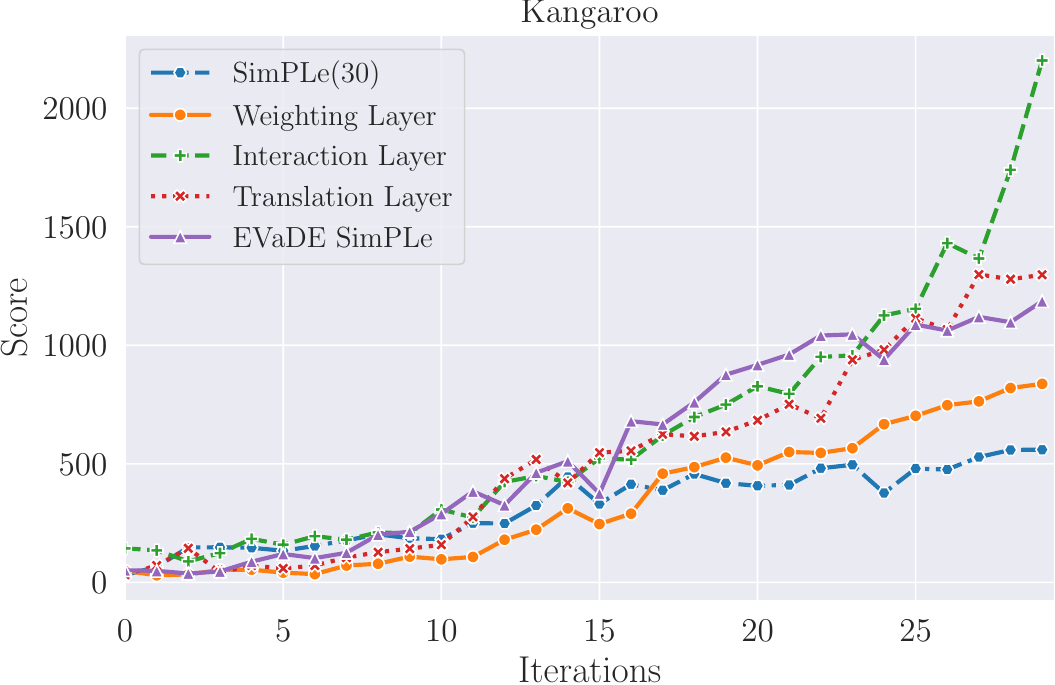}
}
\subfigure{
\includegraphics[width=0.153\textwidth]{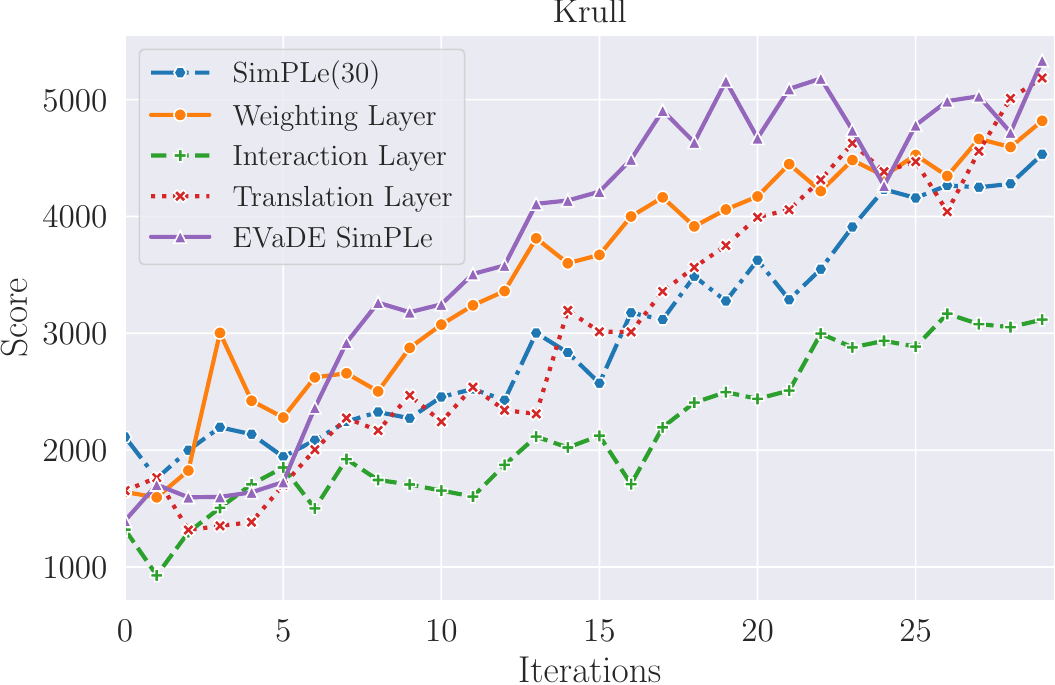}
}
\subfigure{
\includegraphics[width=0.153\textwidth]{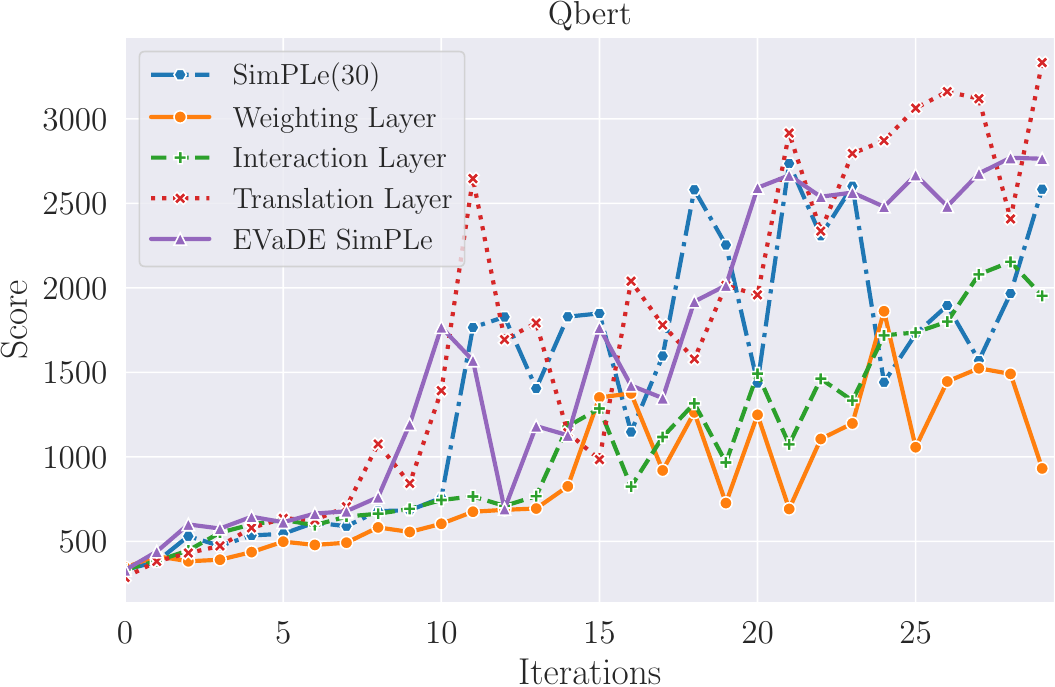}
}
\subfigure{
\includegraphics[width=0.153\textwidth]{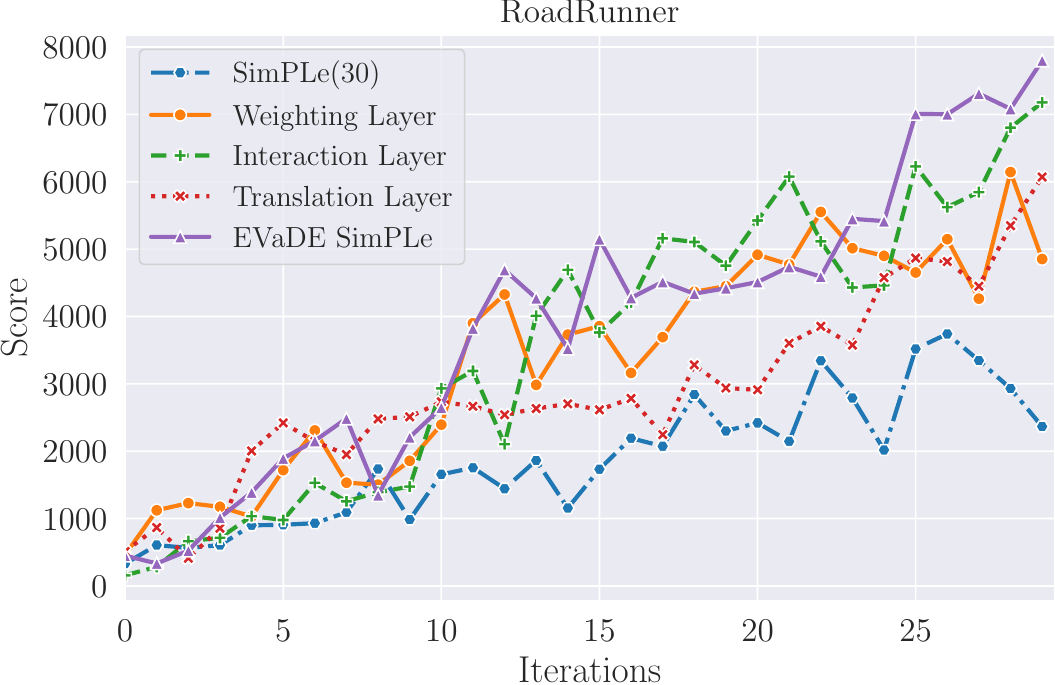}
}
\subfigure{
\includegraphics[width=0.153\textwidth]{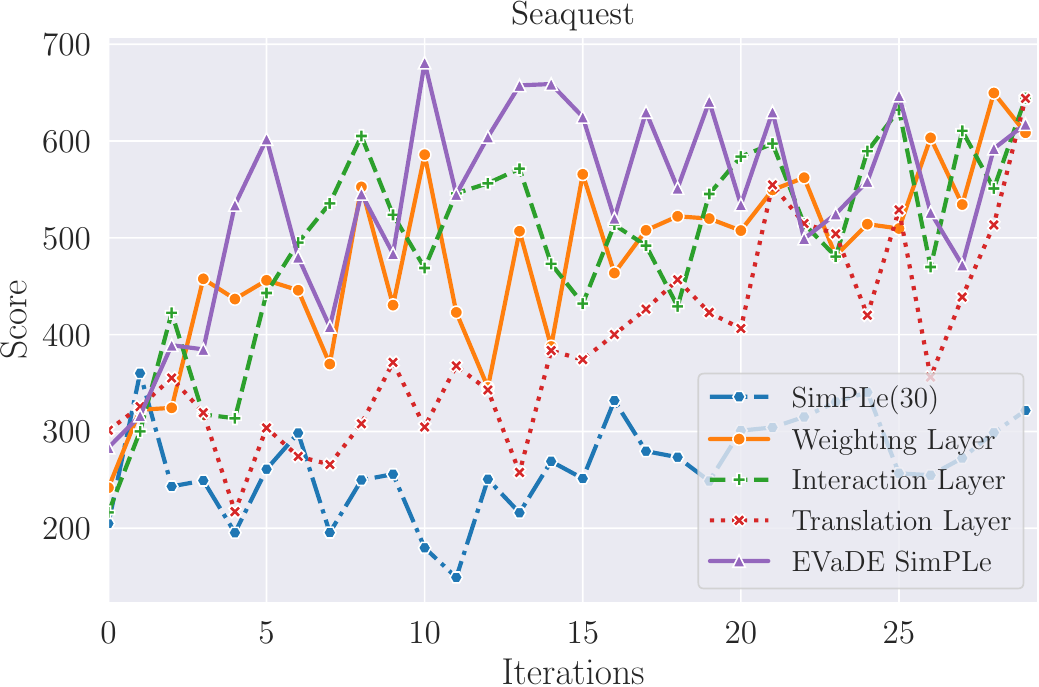}
}
}
\caption{Learning curves of EVaDE-SimPLe agents, SimPLe(30) agents and agents which only add one of the EVaDE layers trained on the 12 game subset of the  Atari 100K test suite.}
\label{fig:evade_exp2}
\end{figure*} 

We present the learning curves of the trained EVaDE and SimPLe(30) agents in  Figure \ref{fig:evade_exp2}. We omit the error bars here for clarity.  Looking at the learning curves presented, it can possibly be said that an increase in scores of SimPLe(30) equipped  with one of the EVaDE layers at a particular iteration  would mean an increase in scores of EVaDE-SimPLe, albeit in later iterations. This pattern can clearly be seen in the games of BankHeist, Frostbite, Kangaroo, Krull and Qbert. This delay in learning could possibly be attributed to the agent draining its interaction budget by exploring areas suggested by one of the layers that are ineffective for that particular game. However, we hypothesise that since all the layers provide different types of exploration, their combination is more often helpful than wasteful. This is validated by the fact that EVaDE-SimPLe achieves higher mean HNS and IQM than any other agent in this study.

We also look at the policies learnt by these agents in RoadRunner, where every EVaDE layer seems to improve the performance of the SimPLe agent even when added individually. We find that the vanilla-SimPLe RoadRunner agent learns a suboptimal policy where it frequently either collides with a moving obstacle or gets caught by the coyote. When the agent adds noisy interaction layers into its reward models, the policy learns to trick the coyote into colliding with the obstacle, an interaction beneficial to the roadrunner. On the other hand, when the reward models use only the noisy translation layers, the agent seems to learn a less \textit{risky} policy, as it aggressively keeps away from both the obstacle as well as the coyote. The agent that adds only the noisy weighting layers seems to prioritize collecting points. This allows the coyote to get near the roadrunner, which could be undesirable. The policy learnt by the EVaDE-SimPLe agent combines the properties of the agents that add only the interaction and translation layers, as it tricks the coyote to colliding with the obstacle, while keeping a safe distance from both. The behaviour of the agent in this game provides some evidence that EVaDE layers can allow us to design different types of exploration based on our prior knowledge. We include videos of one episode run of each agent in the supplementary. 

From Table \ref{tab:EVADE_ablation}, it can be seen that individually, each filter achieves a higher HNS than SimPLe(30), thus indicating that, on average, all the filters help in aiding exploration. Moreover, we see that with the exception of the noisy event interaction layer, the increase in HNS when the other two EVaDE layers are added individually is considerable. Additionally, all agents that add any of the noisy convolutional layers achieve higher IQM scores than baseline SimPLe agents.  Furthermore, we hypothesise that all the layers provide different types of exploration since their combination is more often helpful than wasteful. This is validated by the fact that EVaDE-SimPLe achieves a higher mean HNS and IQM than any other agent in this study.

While having more parameters enlarges the class of functions representable by the model used by EVaDE agents, we emphasize that it is the design of the layers and not the higher number of parameters that is the reason for the performance gains. The agents that include only the noisy translation layers outperform SimPLe(30) on all metrics but add just 4K parameters to the reward model which contains about 10M parameters. Similarly, EVaDE-SimPLe adds only 0.43M parameters for its improvements. 

\subsection{Visualizations of the EVaDE layers}
\label{sec:vis}
We also present some visualizations that help us understand the functionality of the EVaDE layers. All these visualizations were obtained after the completion of training, with the learned weights and variances of the final trained model.

In Figure \ref{fig:nil} we show illustrations of output map of a noisy event interaction layer detecting interactions between the right facing green-coloured enemy ships and the right facing blue-coloured divers given different input images from the game of Seaquest. We also show two input feature maps, which seem to capture the positions of these objects at the same locations. We observe that the pixels in the output feature map in Figure \ref{fig:nil}d are brighter at the locations where the two objects are close to each other, whereas in the same feature map these pixels are dimmer when the two objects are separated by some distance (Figure \ref{fig:nil}h).

\begin{figure}
{\scriptsize
\centering
\hspace*{\fill}
\subfigure[{\scriptsize Input Frame 1}]{
\includegraphics[width=0.15\linewidth, height=2.5cm]{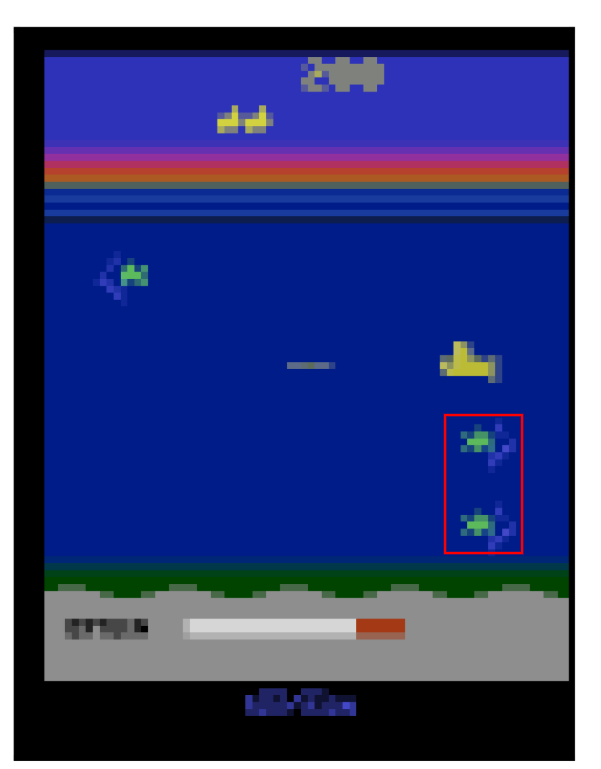}
}
\hfill
\subfigure[{\scriptsize Input Map $1^1$}]{
\includegraphics[ width=0.14\linewidth, height=2.5cm]{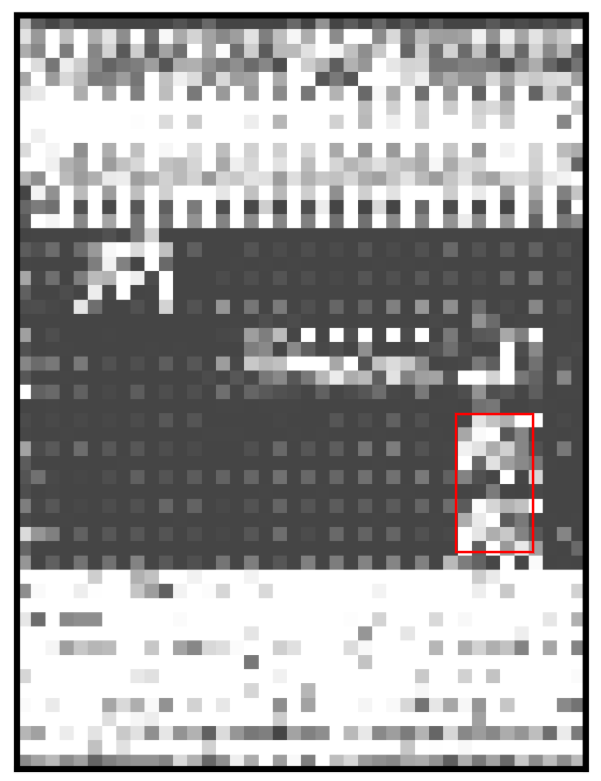}
}
\hfill
\subfigure[{\scriptsize Input Map $1^2$}]{
\includegraphics[ width=0.14\linewidth, height=2.5cm]{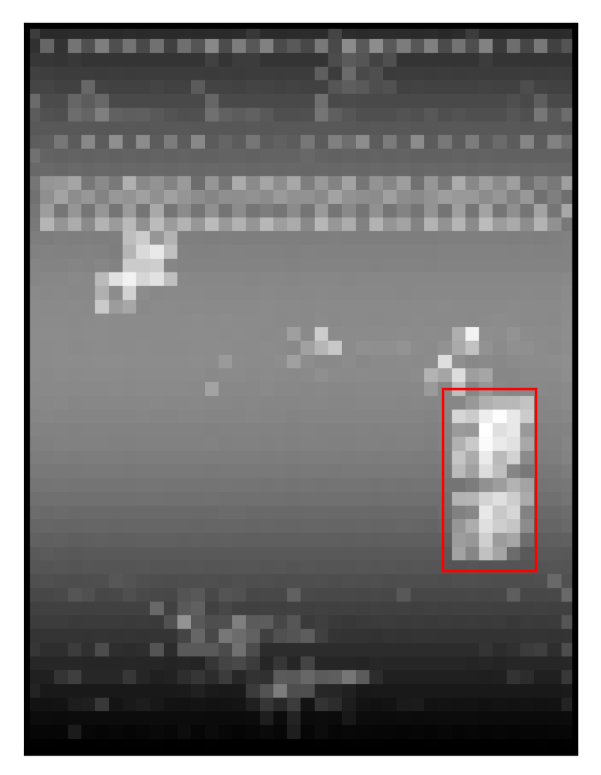}
}
\hfill
\subfigure[{\scriptsize Output Map 1}]{
\includegraphics[ width=0.15\linewidth, height=2.5cm]{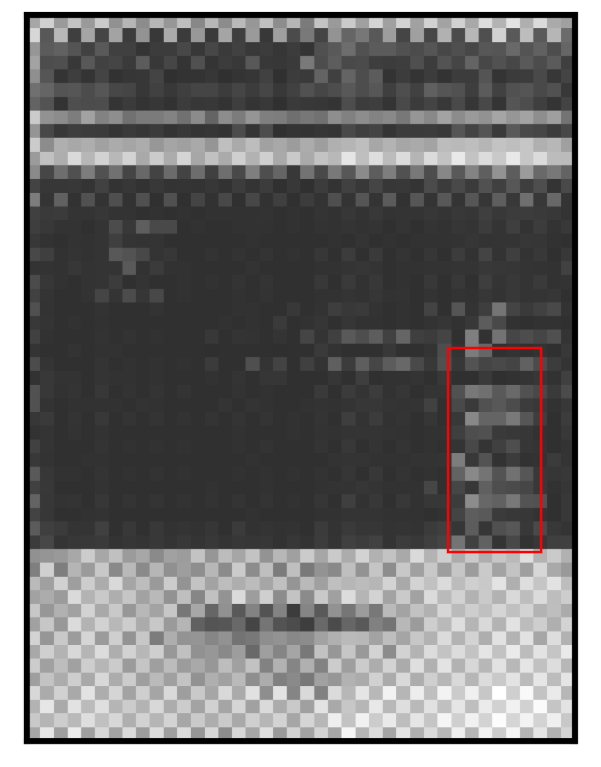}
}
\hspace*{\fill}
\newline
\hspace*{\fill}
\subfigure[\scriptsize Input Frame 2]{
\includegraphics[ width=0.15\linewidth, height=2.5cm]{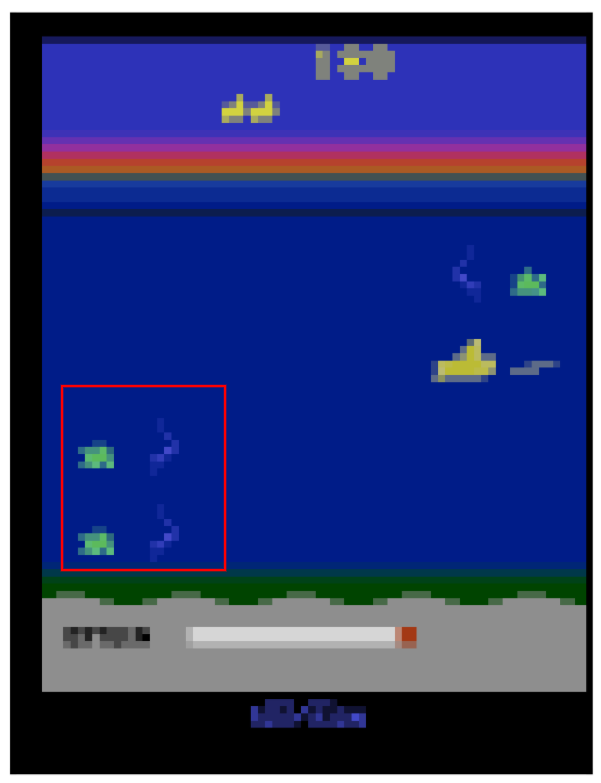}
}
\hfill
\subfigure[\scriptsize Input Map $2^1$]{
\includegraphics[ width=0.14\linewidth, height=2.5cm]{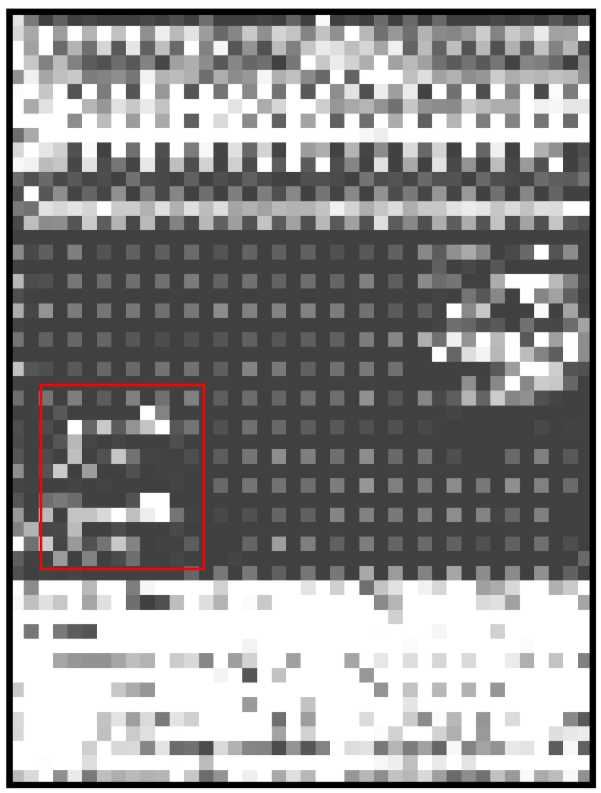}
}
\hfill
\subfigure[\scriptsize Input Map $2^2$]{
\includegraphics[ width=0.14\linewidth, height=2.5cm]{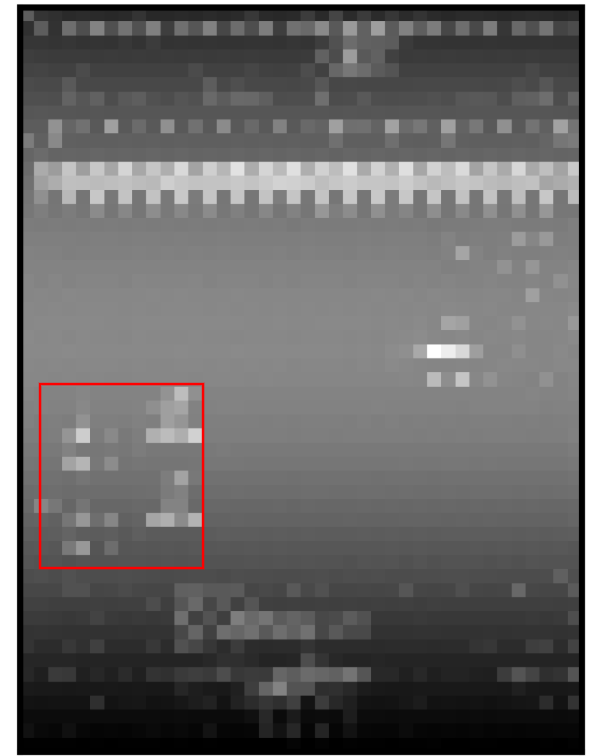}
}
\hfill
\subfigure[\scriptsize Output Map 2]{
\includegraphics[ width=0.15\linewidth, height=2.5cm]{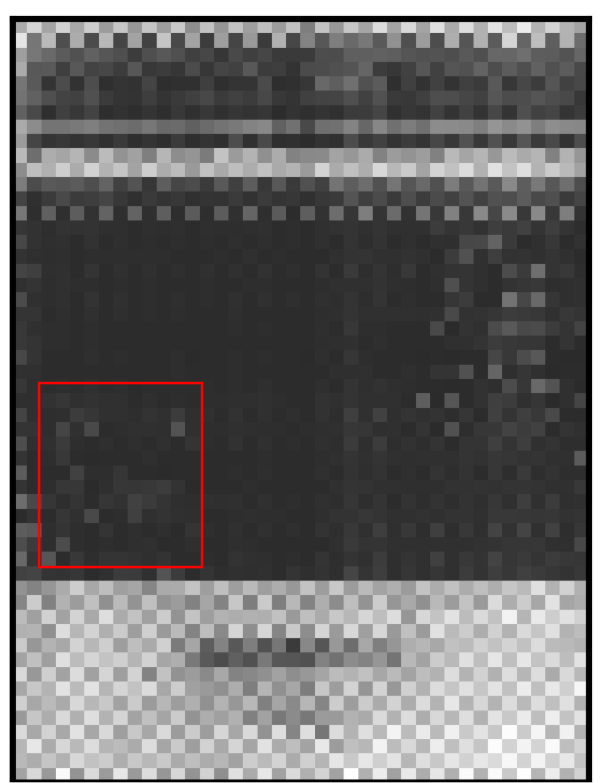}
}
\hspace*{\fill}
}
\caption{This figure shows the output map that captures interactions between two input maps when passed through the noisy event interaction layer.}

\label{fig:nil}
\end{figure}

\begin{figure}
\centering
{\scriptsize
\hspace*{\fill}
\subfigure[{\scriptsize Input Frame}]{
\includegraphics[width=0.14\linewidth, height=2.5cm]{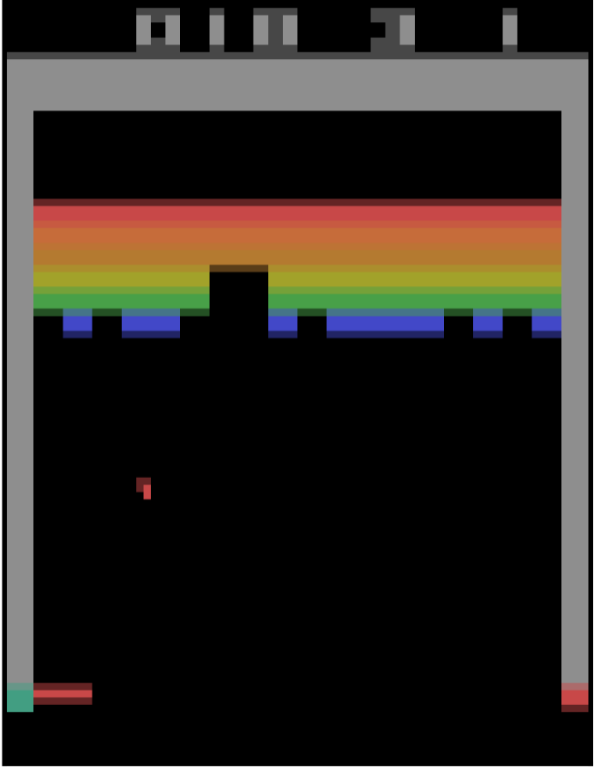}
}
\hfill
\subfigure[{\scriptsize Input Map}]{
\includegraphics[width=0.14\linewidth, height=2.5cm]{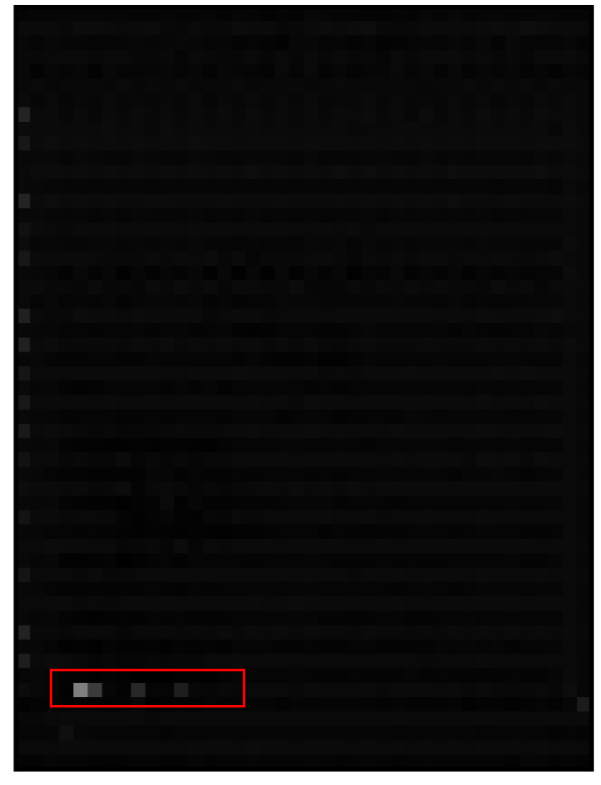}
}
\hfill
\subfigure[{\scriptsize Output Map}]{
\includegraphics[width=0.14\linewidth, height=2.5cm]{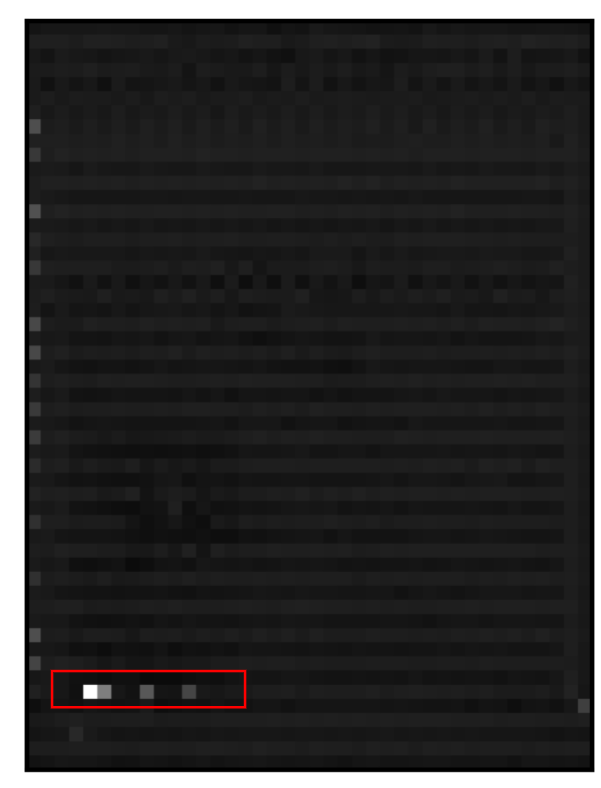}
}
\hspace*{\fill}
}
\caption{This figure shows an output map (channel) that up-weights the corresponding input feature map when passed through the noisy event weighting layer.}
\label{fig:nwl_up}
\end{figure}
\begin{figure}
\centering
{\scriptsize
\hspace*{\fill}
\subfigure[{\scriptsize Input Frame}]{
\includegraphics[width=0.14\linewidth, height=2.5cm]{figures/visualizations/Breakout_weighting_1_up.png}
}
\hfill
\subfigure[{\scriptsize Input Map}]{
\includegraphics[width=0.14\linewidth, height=2.5cm]{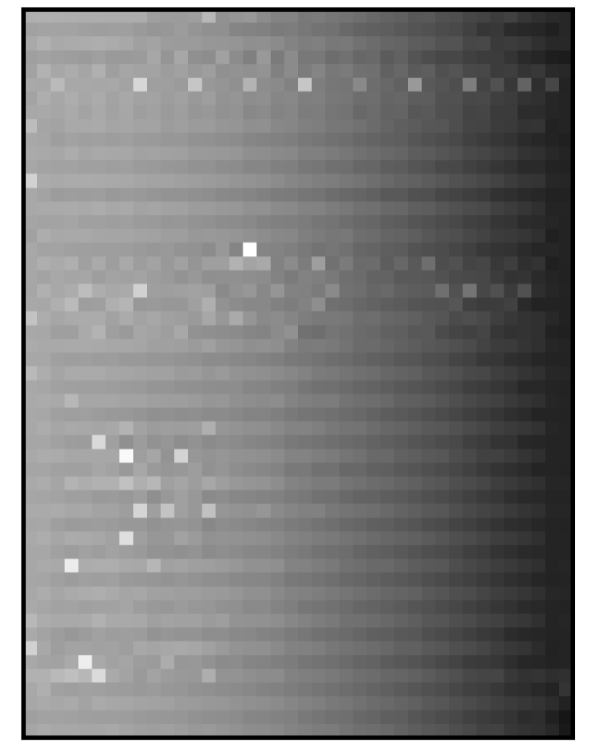}
}
\hfill
\subfigure[{\scriptsize Output Map}]{
\includegraphics[width=0.14\linewidth, height=2.5cm]{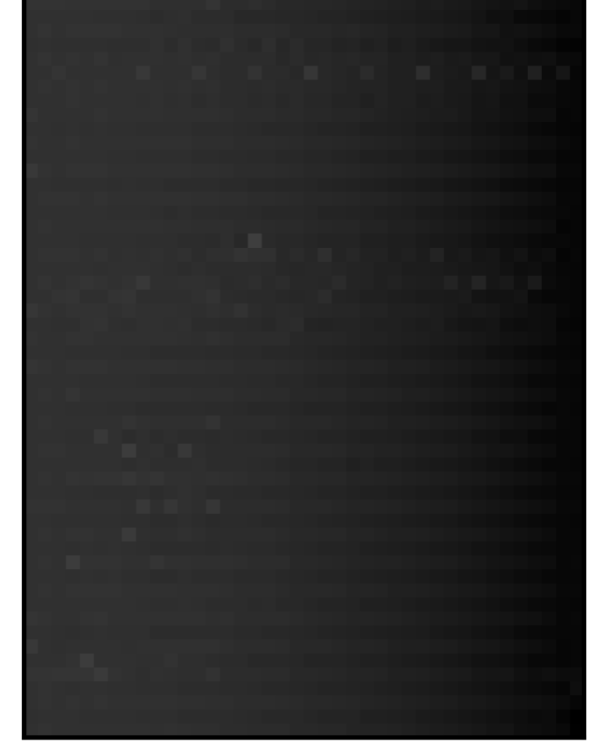}
}
\hspace*{\fill}
}
\caption{This figure shows an output map (channel) that down-weights the corresponding input feature map when passed through the noisy event weighting layer.}
\label{fig:nwl_down}
\end{figure}
In Figures \ref{fig:nwl_up} and \ref{fig:nwl_down}, we show two feature maps before and after passing them through a noisy event weighting layer. The input images for these visualizations were taken from the game of Breakout. The input feature maps to the noisy event weighting layer seem to capture the objects from the input image. The output feature map in Figure \ref{fig:nwl_up} is an upweighted version of its input, as the pixels seem to be brighter. On the other hand, the output feature map in Figure \ref{fig:nwl_down} seems to down-weight its input feature map, as the pixels seem a lot dimmer. The weighting factors for the input-output pairs shown in Figures \ref{fig:nwl_up} and \ref{fig:nwl_down} are 1.93 and 0.57 respectively. 
In Figure \ref{fig:ntl}, we show the input and output feature maps of a game state from Krull, before and after passing it through a noisy event translation layer. The input feature map seems to capture different objects from the input image. The translation effect in output feature map can be seen clearly as every light pixel in the input seems to have lightened up the pixels to its top, bottom, left and right to different degrees.

%% file: conclusion.tex
\section{Conclusion}
In this paper, we introduce Event-based Variational Distributions for Exploration (EVaDE), a set of variational distributions over reward functions. EVaDE consists of three noisy convolutional layers: the noisy event interaction layer, the noisy event weighting layer, and the noisy event translation layer. These layers are designed to generate trajectories through parts of the state space that may potentially give high rewards, especially in object-based domains. We can insert these layers between the layers of the reward network models, inducing variational distributions over the model parameters. The dropout mechanism then generates perturbations on object interactions, importance of events, and positional importance of objects/events. We draw samples from these variational distributions and generate simulations to train the policy of a Simulated Policy Learning (SimPLe) agent.

\begin{figure}
\centering
{\scriptsize
\hspace*{\fill}
\subfigure[{\scriptsize Input Frame}]{
\includegraphics[width=0.14\linewidth, height=2.5cm]{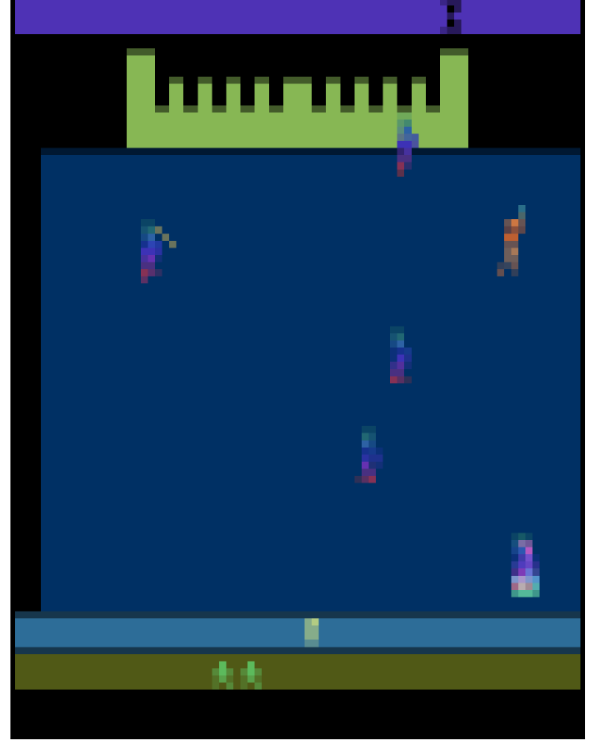}
}
\hfill
\subfigure[{\scriptsize Input Map}]{
\includegraphics[width=0.14\linewidth, height=2.5cm]{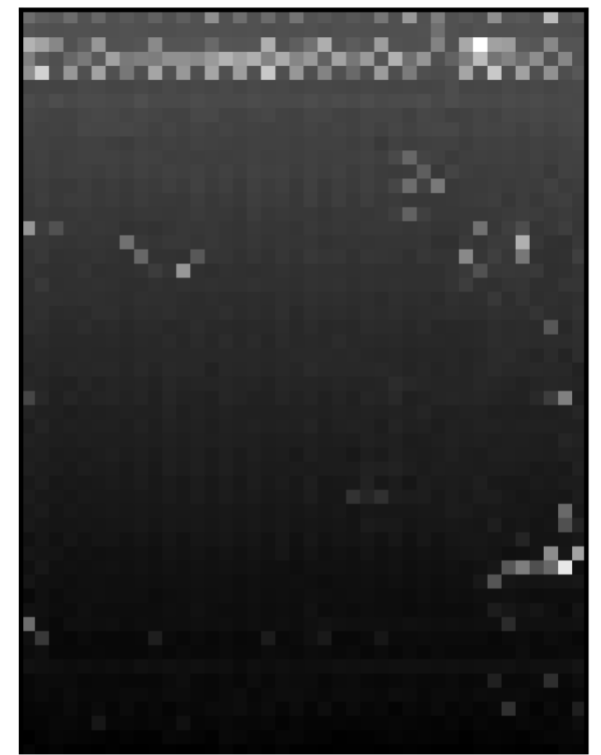}
}
\hfill
\subfigure[{\footnotesize Output Map}]{
\includegraphics[width=0.14\linewidth, height=2.5cm]{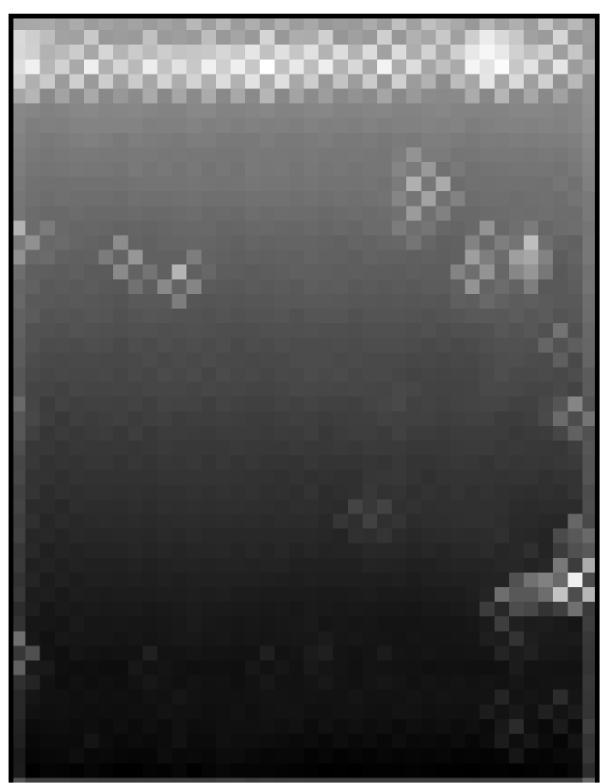}
}
\hspace*{\fill}
}
\caption{This figure shows the function of the noisy translation layer. The output map translates the input pixels to its top, bottom, left and right to different degrees}
\label{fig:ntl}
\end{figure}

We conduct experiments on the Atari 100K test suite, a test suite comprising 26 games where the agents are only allowed 100K interactions with the real environment. EVaDE-SimPLe agents outperform vanilla-SimPLe(30) on this test suite by achieving a mean HNS of 0.682, which is 30\% more than the mean HNS of 0.525 achieved by vanilla-SimPLe(30) agents. EVaDE-SimPLe agents also surpass human performances in five games of this suite. Additionally, these agents also outperform SimPLe(30) on median HNS as well as IQM scores, which is a  metric that is resilient to outlier runs and games. Furthermore, a paired-t test also confirms the statistical significance of the improvements in HNS on the test suite ($p = 3 \times 10^{-3}$). We also find, through an ablation study, that each noisy convolutional layer, when added individually to SimPLe results in a higher mean HNS and IQM. Additionally, the three noisy layers complement each other well, as EVaDE-SimPLe agents, which include all three EVaDE layers,  achieve higher mean HNS and IQM than agents which add only one noisy convolutional layer. Finally, we also present visualizations that help us understand the functionality of the EVaDE layers.






%% file: algorithm.tex
\section{EVaDE-SimPLe as an approximation of PSRL}
\begin{algorithm*}[hbtp!]
    \caption{Approximate PSRL with EVaDE equipped Simulated Policy Learning}
     Initialize agent policy, environment model and reward model dropout parameters $\theta_\pi, \theta_{env}, \sigma_{rew}$ respectively;
     
    Initialize empty real environment interaction dataset  $D_{real} \leftarrow \{\}$;
    
    \For{iteration in $1 \cdots T$}{
        $s \leftarrow \emptyset$\;
        
        \tcp{Interact with real environment}      
        \While{$k_{real}$ real-world interactions not collected}{
            \If{$s$ is terminal or $\emptyset$}{
                Start new episode, initialize start state $s$\;
            }
            Sample action $a \sim \text{Policy}(s, \theta_\pi)$\;
            
            $(s',r) \leftarrow \text{Interact\_Real\_World}(s, a)$\;
            
            $D_{real} \leftarrow D_{real} \cup \{(s,a,r,s')\}$, $s \leftarrow s'$\;
        }
        \tcp{Learn a variational posterior}
        $\theta_{env}, \sigma_{rew} \leftarrow \text{Supervised\_Learn}(\theta_{env}, \sigma_{rew}, D_{real})$\;
        
        \tcp{Draw a sample from the posterior}
        \For{layer $i$ in environment model}{
            \If{$i$ is an EVaDE layer}{
                Sample $\epsilon^i  \sim N(0,1)$\;
                
                $\Tilde{\theta}^i_{env} \leftarrow \theta^i_{env}(1 + \sigma^i_{rew} \epsilon^i)$\;
            }
            \Else{
                $\Tilde{\theta}^i_{env} \leftarrow \theta^i_{env}$\;
            }
        }

        \tcp{Train policy with environment sample}
        $s \leftarrow \emptyset$, $D_{sim} \leftarrow \emptyset$, steps $\leftarrow 0$\;
        
        \While{$k_{sim}$ interactions not completed}{
            \If{$s$ is terminal or $\emptyset$}{
                Start new episode, initialize start state $s$\;
            }
            Sample action $a \sim \text{Policy}(s, \theta_\pi)$\;
            
            $(s',r) \leftarrow \text{Interact\_Env\_Sample}(\Tilde{\theta}_{env},s, a)$\;
            
            $D_{sim} \leftarrow D_{sim} \cup \{(s,a,r,s')\}$, $s \leftarrow s'$, steps $\leftarrow$ steps + 1\;
            
            \If{steps $\bmod$ update\_frequency = 0}{
                Update policy: $\theta_\pi \leftarrow \text{Reinforcement\_Learn}(\theta_\pi,D_{sim})$\;
            }
        }
    }
    \Return $\theta_\pi$\;
    \label{algo:evade}
\end{algorithm*}

We present the pseudocode of EVaDE-SimPLe in Algorithm \ref{algo:evade}. As mentioned in Section \ref{sec:evade}, an EVaDE-SimPLe agent has the same iterative training structure as SimPLe and PSRL. In the first step of each iteration, the agent interacts with the real environment using its latest policy to collect interactions. The agent then updates its posterior distribution over the environment model parameters by jointly optimizing the environment model parameters $\theta_{env}$ which include the parameters of the transition and reward function and the Gaussian dropout parameters of the reward network $\sigma_{rew}$ with the help of supervised learning. A sample from this approximate posterior distribution is then acquired with the help of Gaussian dropout. Subsequently, the agent updates its policy by optimizing the parameters of the policy network, $\theta_\pi$ by interacting with this environment sample. This optimized policy is used by the agent to procure more training interactions by interacting with the

%% file: theorem_proof.tex
\section{Proof of Theorem \ref{th}}
\label{sec:proof}
We provide the proof for Theorem \ref{th}, which is restated below, in this section. 

\textbf{Theorem \ref{th} :} Let $\mathbbm{n}$ be any neural network. For any convolutional layer $l$, let $m_i(l) \times n_i(l) \times c_i(l)$ and $m_o(l) \times n_o(l) \times c_o(l)$ denote the dimensions of its input and output respectively. Then, any function that can be represented by $\mathbbm{n}$ can also be represented by any network $\Tilde{\mathbbm{n}} \in \Tilde{\mathscr{N}}$, where $\Tilde{\mathscr{N}}$ is the set of all neural networks that can be constructed by adding any combination of EVaDE layers to  $\mathbbm{n}$, provided that, for every EVaDE layer $\Tilde{l}$ added, $\Tilde{l}$ uses a  stride of $1$, $m_i(\Tilde{l}) \leq m_o(\Tilde{l}), n_i(\Tilde{l}) \leq n_o(\Tilde{l})$ and $c_i(\Tilde{l}) \leq c_o(\Tilde{l})$.

\subsection{Notations}
\subsubsection*{Neural Networks}
Any function $f$ represented by a $k$-layer neural network $\mathbbm{n}$ is an ordered composition of the functions $f_1,f_2,\cdots, f_k$ computed by its layers  $N_1, N_2, \cdots N_k$ respectively, i.e., $f = f_k \circ f_{k-1}  \circ \cdots  f_1$.  

\subsubsection*{Convolutional Layers}
Any $m \times n$ convolutional layer $l$ has a total of $m \times n \times c_i(l) \times c_o(l)$ learnable parameters, where $c_i(l)$ and $c_o(l)$ are the number of channels in the input and output of the layer respectively. The parameters of any convolutional layer $l$ can be partitioned into $c_o(l)$ filters, where each filter has $m \times n \times c_i(l)$ parameters, and is responsible for computing one output channel.

We denote the set of parameters of any convolutional layer by $\theta$. We denote the set of parameters of the $k^{th}$ filter by $\theta_k$, and the parameters of the $l^{th}$ channel of this filter by $\theta^l_k$. We denote the $(i,j)^{th}$ parameter of the $l^{th}$ channel of the $k^{th}$ filter by $\theta^{l,i,j}_k$. 
For noisy convolutional layers, we have a learnable Gaussian dropout parameter attached to every parameter of the convolutional layer (see Equation \ref{eq:reparameterization}). We use $\sigma_k$, $\sigma^l_k$ and $\sigma^{l,i,j}_k$ to denote the dropout parameters of the $k^{th}$ filter, the $l^{th}$ channel of the $k^{th}$ filter and the $(i,j)^{th}$ parameter of the $l^{th}$ channel of the $k^{th}$ filter respectively.

\subsubsection*{Strides}
A stride is a hyperparameter of a convolutional layer, that determines the number of pixels of the input that each convolutional filter moves, to compute the next output pixel. 

\subsection{Implications of the constraints in Theorem \ref{th}}
Theorem \ref{th} states that every EVaDE layer $\Tilde{l}$ added uses a stride of 1 and satisfies the constraints  $m_i(\Tilde{l}) \leq m_o(\Tilde{l}), n_i(\Tilde{l}) \leq n_o(\Tilde{l})$ and $c_i(\Tilde{l}) \leq c_o(\Tilde{l})$. This means that for any inserted EVaDE layer, every output dimension is at least as large as its corresponding input dimension. This eventually implies that for every EVaDE layer, all input and output dimensions match, i.e., $m_i(\Tilde{l}) = m_o(\Tilde{l}), n_i(\Tilde{l}) = n_o(\Tilde{l})$ and $c_i(\Tilde{l}) = c_o(\Tilde{l})$. 

To see why, let us assume that the EVaDE layers $\Tilde{l_j}, \cdots \Tilde{l_k}$ are inserted, in order, in between the layers $N_i$ and $N_{i+1}$ of a neural network $\mathbbm{n}$. As  $N_i$ and $N_{i+1}$ are two consecutive layers of $\mathbbm{n}$, we must have $m_i(N_{i+1}) = m_o(N_{i}), n_i(N_{i+1}) = n_o(N_{i})$ and $c_i(N_{i+1}) = c_o(N_{i})$. This implies that the dimensions of the input to layer $\Tilde{l_j}$ match the dimensions of the output of the layer $\Tilde{l_k}$, i.e., $m_i(\Tilde{l_j}) = m_o(\Tilde{l_k}), n_i(\Tilde{l_j}) = n_o(\Tilde{l_k})$ and $c_i(\Tilde{l_j}) = c_o(\Tilde{l_k})$. However, under the constraints imposed in Theorem \ref{th}, every output dimension is greater than or equal to its corresponding input dimension for every EVaDE layer. Thus, matching the output dimensions of $\Tilde{l_k}$ with the input dimensions of $\Tilde{l_j}$ is only possible if the input and output dimensions match for every EVaDE layer $\Tilde{l_j}, \cdots \Tilde{l_k}$ that is added. 

With the above implications, the constraint of using a stride of 1, forces SAME padding for every EVaDE layer, and also ensures that patches centred around every $(i,j)^{th}$ pixel of every channel in the input are used to compute the outputs. This is an important implication that will help us prove the claims that all EVaDE layers can represent the identity transformation. 

\subsection{Claims}

We prove the three following claims by construction, i.e., showing that there is a combination of parameters using which these layers can perform the identity transformation. 
\newline

\begin{claim}
The noisy event interaction layer can represent the identity transformation. 
\end{claim} 
\begin{proof}
Let us assume an $m \times m$ noisy event interaction layer. With the help of the observations made in the previous section, we are ensured of using patches centred around every input $x_{i,j}^{l} \; \forall i,j,l$ and the constraints also ensure that the number of filters in this layer is equal to the number of input channels. 

The identity transformation can be achieved with the following parameter assignments. 
\begin{itemize}
    \item The dropout parameter $\sigma^{l,i,j}_k$ corresponding to every convolutional layer parameter $\theta^{l,i,j}_k$ is set to zero.
    \item The layer parameter corresponding to the central entry of the $k^{th}$ layer of the $k^{th}$ convolutional filter, i.e., $\theta^{k,\ceil{\frac{m}{2}},\ceil{\frac{m}{2}}}_{k}$ is set to 1 $\forall k$.
    \item All other convolutional layer parameters are set to 0. 
\end{itemize}

As stated in Equation \ref{eq:ev_int}, the event interaction layer computes the outputs $y^k_{i,j} \; \forall i,j,k$ using the following equation. 
$$y_{i,j}^k  = \sum\limits_{l=0}^c \mathbbm{1}_{m}^{T} \left(\Tilde{\theta}^{l}_{k} \odot P_{x_{i,j}^{l}}\right)\mathbbm{1}_{m}$$

Applying the above parameter assignments, we get $y_{i,j}^k = x_{i,j}^k$, as the only non-zero parameter in the $k^{th}$ filter, that is set to 1, aligns with $x_{i,j}^k$. This is the required identity transformation.

\end{proof}

\begin{claim}
The noisy event weighting layer can represent the identity transformation. 
\end{claim} 
\begin{proof}
The noisy event weighting layer uses $c$ $1 \times 1$ convolutional filters, where $c$ is the number of input channels. Consequently, $\theta^k_k$, is just a single trainable parameter instead of a grid of trainable parameters as in the other two EVaDE layers. 

The identity transformation can be achieved with the following parameter assignments. 
\begin{itemize}
    \item The dropout parameter $\sigma^{l}_k$ corresponding to every convolutional layer parameter $\theta^{l}_k$ is set to zero.
    \item The layer parameter corresponding to the $k^{th}$ layer of the $k^{th}$ convolutional filter, i.e., $\theta^{k}_{k}$ is set to 1 $\forall k$.
    \item All other convolutional layer parameters are set to 0. 
\end{itemize}
This is a valid assignment, as the only parameters set to 1 are trainable, while the other parameters are forced to be set to 0 by construction (see Section \ref{sec:ev_we}).

As stated in Equation \ref{eq:ev_we}, the event interaction layer computes every output $y^k_{i,j}$ using the following equation. 
$$y_{i,j}^k  =  \Tilde{\theta}^{k}_{k}x_{i,j}^{k}$$

Setting $\theta^k_k = 1$ and $\sigma^k_k = 0 \; \forall k$, yields  $y_{i,j}^k = x_{i,j}^k \; \forall i,j,k$ , which is the identity transformation required. 
\end{proof}

\begin{claim}
The noisy event translation layer can represent the identity transformation. 
\end{claim} 
\begin{proof}

In this case, we can use the parameter assignments as stated in the proof of Claim 1 to produce an identity transformation. This is possible, since we construct the noisy event translation layer with the same structure of an $m \times m$ convolutional layer with the number of filters equalling the number of input channels. Moreover, the only non-zero parameter (which is set to 1) in the $k^{th}$ filter, $\theta^{k,\ceil{\frac{m}{2}},\ceil{\frac{m}{2}}}_{k}$ is in the middle row and middle column of its $k^{th}$ channel, making it a valid assignment for the noisy event translation layer  (see Section \ref{sec:ev_tr}).

As stated in Equation \ref{eq:ev_tr}, the event interaction layer computes every output $y^k_{i,j}$ using the following equation. 
$$y_{i,j}^k  = \mathbbm{1}_{m}^{T} \left( \Tilde{\theta}^{k}_{k} \odot P_{x_{i,j}^{k}}\right)\mathbbm{1}_{m}$$

As in the case with the noisy event interaction layer, substituting these assignments, we get $y_{i,j}^k = x_{i,j}^k \; \forall i,j,k$ , which is the identity transformation required. 

\end{proof}

\subsubsection{Proof of Theorem \ref{th}}

We have to prove that all elements from $\Tilde{\mathscr{N}}$, the set of neural networks that can be constructed by adding any combination of EVaDE layers to the neural network $\mathbbm{n}$, can represent the functions represented by k-layered neural network $\mathbbm{n}$. 

Let $\Tilde{n}$ be a general element from $\Tilde{\mathscr{N}}$, that adds the EVaDE layers $\Tilde{l_1}, \Tilde{l_2}, \cdots \Tilde{l_m}$, in order, after the layers $N_{i_1}, N_{i_2}, \cdots N_{i_m}$ of the neural network $\mathbbm{n}$, where $i_{j-1} \leq i_j \leq i_{j+1} \; ; \forall 2\leq j \leq m-1$ and $i_1 \geq 0, i_m \leq k$. Adding an EVaDE layer after $N_0$ refers to it being added after the input layer and before the first layer of $\mathbbm{n}$. Note that more than one EVaDE layer can be added after any layer $N_j$ of $\mathbbm{n}$. 

Also, let $f_1, f_2, \cdots f_k$ be the functions computed by the layers $N_1, N_2, \cdots, N_k$ of $\mathbbm{n}$ respectively. Thus the function represented by $\mathbbm{n}$ is $f = f_k \circ f_{k-1}  \circ \cdots  f_1$. 

Let  $\Tilde{f_1}, \Tilde{f_2}, \cdots \Tilde{f_m}$ be the functions computed by the EVaDE layers $\Tilde{l_1}, \Tilde{l_2}, \cdots \Tilde{l_m}$ respectively. Thus the function computed by $\Tilde{n}$ is $\Tilde{f} = f_k \circ f_{k-1}  \circ \cdots \circ \Tilde{f}_m \circ f_{i_m} \cdots \circ \Tilde{f_1} \circ  f_{i_1} \circ \cdots f_1$. As all $\Tilde{f_1}, \Tilde{f_2}, \cdots \Tilde{f_m}$ can learn to represent the identity transformation, $\Tilde{f}$ can learn to represent $f$. This implies that $\Tilde{n}$ can represent any function represented by $\mathbbm{n}$.

%% file: var_dropout.tex
\section{Variational Distributions using Dropouts}
Variational methods are used to approximate inference and/or sampling when using  intractable posterior distributions. These methods work by using variational distributions that facilitate easy sampling and/or inference, while approximating the true posterior as closely as possible.

These methods require the user to define two distributions, the prior $p(\theta)$, and the variational distribution $q(\theta)$. Given a set of training samples $D = (X,Y)$, where $X$ is the set of input samples and $Y$ the set of corresponding labels, variational methods work to minimize the KL-divergence between the learnt variational distribution $q(\theta)$ and the  true posterior $p(\theta | D)$. This is equivalent to maximizing the Evidence Lower Bound (ELBO) as shown below. 

\begin{align*}
KL(q(\theta), p(\theta|D)) = \int q(\theta) \log \frac{q(\theta)}{p(\theta|D)} d\theta
\end{align*}

Now, 
\begin{align*}
p(\theta|D) =& p(\theta|X,Y) = \frac{p(\theta)p(X,Y|\theta)}{p(X,Y)} = \frac{p(\theta)p(Y|X,\theta)p(X|\theta)}{p(X,Y)} \\
=& \frac{p(\theta)p(Y|X,\theta)p(X)}{p(X,Y)};  \end{align*}

Substituting the value for $p(\theta|D)$, 
\begin{align*}
&KL(q(\theta), p(\theta|D)) = \int q(\theta) \left[ \log \frac{q(\theta)p(X,Y)}{p(\theta)p(Y|X,\theta)p(X)} \right] d\theta \\
&= \int q(\theta) \log \frac{q(\theta)}{p(\theta)p(Y|X,\theta)}d\theta + \int q(\theta) \log \frac{  P(X,Y)}{P(X)} d\theta \\
&= \int q(\theta) \log \frac{q(\theta)}{p(\theta)p(Y|X,\theta)} d\theta +  \log \frac{  P(X,Y)}{P(X)} \\
&= \int q(\theta) \log \frac{q(\theta)}{p(\theta)}d\theta - \int q(\theta) \log p(Y|X,\theta)d\theta +  \log \frac{  P(X,Y)}{P(X)} \\
&= KL( q(\theta), p(\theta)) - \int q(\theta) \log p(Y|X,\theta)d\theta +  \log \frac{  P(X,Y)}{P(X)}
\end{align*}
 Since $P(X,Y)$ and $P(X)$ are constants with respect to $\theta$, the set of parameters that minimize $KL(q(\theta), p(\theta|D))$ are the same as the ones that maximize the ELBO, i.e., 
 
 $$\argmin_{\theta} KL(q(\theta), p(\theta|D)) = \argmax_{\theta}  \int q(\theta) \log p(Y|X,\theta)d\theta -KL( q(\theta), p(\theta))$$

\subsection{Dropouts as Variational Distributions}
\citep{gal2016dropout} introduces the usage of dropout as a mechanism to induce variational distributions, samples from which are used to approximate the ELBO.  The first term of the ELBO can be re-written as,
$$\int q(\theta) \sum\limits_{i=1}^{N}\log p(y_i | x_i, \theta)d\theta$$
where every $(x_i,y_i)$ is a training example from $D$.

This integral can be approximated by averaging out the log-probabilities using several samples drawn from the variational distribution $q(\theta)$ (Equation \ref{eq:elbo_approx}).

\begin{equation}
\label{eq:elbo_approx}
\int q(\theta) \sum\limits_{i=1}^{N}\log p(y_i | x_i, \theta)d\theta \approx  \sum\limits_{i=1}^{N}\log p(y_i | x_i, \theta_i); \; \text{where }\theta_i \sim q(\theta)
\end{equation}

Neural networks that use different types of dropouts help us maintain variational distributions $q(\theta)$ that approximate posteriors over deep Gaussian processes \citep{gal2016dropout,kingma2015variational,aravindan2021state}. Procuring a sample from this posterior using $q(\theta)$ is easy, as a random dropped out network corresponds to a  sample from the  posterior over these deep Gaussian processes.

 

%% file: network_architectures.tex
\section{Network Architectures}
\label{sec:network_arch}
In this section, we detail the network architectures used for training the environment models of SimPLe \citep{Kaiser2020Model} and EVaDE-SimPLe, and the policy network architectures used by both the methods.   

\subsection{Environment Network Architecture}
\subsubsection{SimPLe}
In our experiments we use the network architecture of the deterministic world model introduced in \citep{Kaiser2020Model}  to train the environment models of the SimPLe agents, but do not augment it with the convolutional inference network and the autoregressive LSTM unit. 

Given four consecutive game frames and an action as input, the network jointly models the transition and reward functions, as it predicts the next game frame and the reward using the same network. The network consists of a dense layer, which outputs a pixel embedding of the stacked input frames. This layer is followed by a stack of six $4 \times 4$ convolutional layers, each with a stride of 2. These layers are followed by six $4 \times 4 $ de-convolutional layers. For $1 \leq i \leq 5$, the $i^{th}$ de-convolutional layers, take in as input, the output of the previous layer, as well as the output of the $6-i^{th}$ convolutional layer. The last de-convolutional layer takes in as input the output of its previous layer and the dense pixel embedding layer. An embedding of the action input is multiplied and added to the input channels of every de-convolutional layer. The outputs from the last de-convolutional layer is passed through a softmax function to predict the next frame. The outputs from the last de-convolutional layer is also combined with the output of the last convolutional layer and then  passed through a fully connected layer with $128$ units followed by the output layer to predict the reward. 

\subsubsection{EVaDE-SimPLe}
The architecture of the environment model used by EVaDE agents is shown in Figure \ref{fig:evade_arch}. This model resembles the model of SimPLe agents until the fourth de-convolutional layer. 
All the stand-alone EVaDE layers that we use, use a stride of 1 and SAME padding so as to keep the size of the inputs and outputs of the layer same. As the EVaDE layers are added only to the reward function, we split the network into two parts, one that predicts the next frame (the transition network) and one that predicts the reward (the reward network) respectively. We denote the last two de-convolutional layers in each part $d^{t}_5,d^{t}_6$ and $d^{r}_5,d^{r}_6$ respectively.

As shown in Figure \ref{fig:evade_arch}, in the transition network, the outputs of the fourth de-convolutional layer and the first convolutional layer are passed to $d^{t}_5$. $d^{t}_6$ takes in as inputs the outputs of $d^{t}_5$  and the pixel embedding layer. 

The reward network adds a combination of a $3 \times 3$ noisy event translation layer, a noisy event weighting layer and a $1\times 1$ noisy event interaction layer which are inserted before both $d^{r}_5$ and $d^{r}_6$. $d^r_5$ shares weights with $d^{t}_5$, and takes in the outputs of the previous event interaction layer and the first convolutional layer as inputs. Likewise, $d^r_6$ shares weights with $d^{t}_6$, and takes in the outputs of the previous event interaction layer and the pixel embedding layer as inputs. Moreover, we also apply Gaussian multiplicative dropout to the weights of $d^r_6$, to make it act as an event interaction layer. As with SimPLe agents, an embedding of the action input is multiplied and added to the input channels of every de-convolutional layer (also shown in Figure \ref{fig:evade_arch}).

The outputs of $d^t_6$ are passed through a softmax to predict the next frame, while the outputs of $d^r_6$ are combined with the output of the last convolutional layer and  passed through a fully connected layer with $128$ units followed by the output layer to predict the reward. 

\subsection{Policy Network}
The policy network for both SimPLe and EVaDE-SimPLe agents consists of two convolutional layers followed by a hidden layer and an output layer. The inputs to the policy network are four consecutive game frames, which are stacked and passed through two $5 \times 5$ convolutional layers, both of which use a stride of 2. These convolutional layers are followed by a fully connected layer with $128$ hidden units, which is followed by the output layer, that predicts the stochastic policy, i.e., the probabilities corresponding to each valid action, and the value of the current state of the agent.

%% file: experimental_details.tex
\section{Experimental Details}

\subsection{Codebase used and Hyperparameters}
The code for Simple(30) and EVaDE-SimPLe agents is provided in the supplementary.


We build our SimPLe(30) and EVaDE-SimPLe agents by utilizing the implementation of SimPLe agents from \citep{tensor2tensor}. To keep the comparison fair, we use the same hyperparameters as used by \citep{tensor2tensor} to train all our agents. The codebase in \citep{tensor2tensor} uses an Apache 2.0 license, thus allowing for public use and extension of their codebase.






\subsection{Computational Hardware Used}
We train our agents on a cluster of 4 NVIDIA RTX 2080 Ti GPUs with an Intel Xeon Gold 6240 CPU. The total time taken to train 5 independent runs of all 5 algorithms on the test suite of 12 games in addition to 5 independent runs of SimPLe(30) and EVaDE-SimPLe on the rest of the 14 games in the suite was around 195 days (or about 6.5 months).

\subsection{Human Normalized Score}

We use the human normalized scores from \citep{DBLP:journals/corr/abs-1905-12726} as defined in Equation \ref{eq:hns} to compare our agents. 
\begin{equation}
\label{eq:hns}
    \text{HNS}_\text{agent} = \frac{\text{Score}_\text{agent} - \text{Score}_\text{random}} {\text{Score}_\text{human} - \text{Score}_\text{random}}
\end{equation}

where $\text{Score}_\text{agent}, \text{Score}_\text{human}$ and $\text{Score}_\text{random}$ denote the scores achieved by agent being evaluated, a human and an agent which acts with a random policy respectively.

We also list the baseline scores achieved by humans and random agents, as listed in \citep{DBLP:journals/corr/abs-1905-12726} in Table \ref{tab:baseline_scores} for easy access.

\begin{table}[ht]
\centering
    \caption{Baseline human and random values used to calculate Human Normalized Scores}
  \begin{tabular}{c c c}
    \toprule
    Game & Human Score& Random Score\\
    \midrule
    Alien & 7,127.7 & 227.8\\
    Amidar & 1719.5& 5.8\\
    Assault & 742& 222.4\\
    Asterix & 8503.3& 210 \\
    BankHeist & 753.1&14.2 \\
    BattleZone & 37187.5&2360 \\
    Boxing & 12.1 & 0.1 \\
    Breakout & 30.5&1.7 \\
    ChopperCommand & 7387.8& 811\\
    CrazyClimber & 35829.4&10780.5 \\
    DemonAttack & 1971& 152.1\\
    Freeway & 29.6& 0 \\
    Frostbite &4334.7&65.2 \\
    Gopher & 2412.5& 257.6\\
    Hero & 30826.4& 1027 \\
    JamesBond &302.8& 29\\
    Kangaroo & 3035 &52\\
    Krull & 2665.5&1598 \\
    KungFuMaster & 22736.3&258.5 \\
    MsPacman &6951.6&307.3 \\
    Pong &14.6& -20.7\\
    PrivateEye &69571.3&24.9 \\
    Qbert & 13455&163.9 \\
    RoadRunner & 7845&11.5 \\
    Seaquest & 42054.7&68.4 \\
    UpNDown &11693.2&533.4 \\
    \bottomrule
  \end{tabular}
    \label{tab:baseline_scores}
\end{table}



\subsection{Inter-Quartile Mean}

Benchmarking the results of reinforcement learning algorithms is inherently noisy, as the results of most training runs of these algorithms depend on a variety of factors including random seeds, choice of the evaluation environment and the codebase used by these runs \citep{henderson2018deep}. While the human normalized scores will average out the variability in the performances of these training runs with a large number of training runs, often these scores are skewed by outlier games or scores, i.e., games or random trials in which the algorithm achieves unusually high or low scores. 

The inter-quartile mean \citep{agarwal2021deep} (IQM) of a reinforcement learning algorithm that is evaluated on $n$ tasks, with $m$ evaluation runs per task, can be computed as the mean of the human normalised scores of those training runs that comprise the 25 - 75 percentile range of these $n \times m$ training runs. In doing so, this metric judges the algorithm on the group of games as a whole, while ignoring the outliers.

\subsection{More Experimental Details}

We present the scores achieved by all five independent runs of all agents trained on the 12-game subset of the Atari 100K test-suite in Table \ref{tab:all_scores}. 
Additionally, we also present the learning curves with  error bars equal to a width of 1 standard error on each side are shown in Figure \ref{fig:evade_exp2_err}.

We present the scores of all five independent runs of EVaDE-SimPLe and SimPLe(30) agents trained on rest of the 14 games in the 100K test-suite  in Table \ref{tab:all_scores_14_game} and in Table \ref{tab:all_mean_scores},  we present the mean scores achieved by SimPLe(30), EVaDE-SimPLe and other baselines in the Atari 100K test-suite. 

shows the learning curves as shown in Figure \ref{fig:evade_exp2} with error bars equal to a width of 1 standard error on each side.

Looking at the learning curves presented in Figure \ref{fig:evade_exp2}, it can possibly be said that an increase in scores of SimPLe(30) equipped  with one of the EVaDE layers at a particular iteration  would mean an increase in scores of EVaDE-SimPLe, albeit in later iterations. This pattern can clearly be seen in the games of BankHeist, Frostbite, Kangaroo, Krull and Qbert. This delay in learning could possibly be attributed to the agent wasting its interaction budget exploring areas suggested by one of the layers that is ineffective for that particular game. However, we hypothesise that since all the layers provide different types of exploration, their combination is more often helpful than wasteful. This is validated by the fact that EVaDE-SimPLe achieves higher mean HNS, IQM and SimPLe-NS than any other agent in this study.

{
\onecolumn
\scriptsize
\begin{longtable}{@{}cccccc@{}}
\caption{Scores achieved by every independent run of every SimPLe agent and when equipped with different EVaDE layers in the 12 game subset of the Atari 100K test-suite.} \label{tab:all_scores}
\\

\toprule
Game&SimPLe(30) & Inter. Layer & Weight. Layer & Trans. Layer & EVaDE-SimPLe\\  \midrule

\multirow{5}{*}{BankHeist} & 133.1 & 85 & 232.2&218.4 &155.3\\ 
 &  9.375& 12.5&195.3&128.8 &205.9\\
  &  13.13& 186.9&154.7&158.4 &347.8\\
 & 69.38 & 142.8 & 130.6 & 187.5 & 250.9\\
 &167.8 & 110.3 & 129.4 & 210.6 & 160.9 \\\midrule
  
\multirow{5}{*}{BattleZone} & 4156 & 1313& 9250&4438 &10844\\ 
 & 6969 & 8031&4250& 6000&9375\\
  & 3344 &9781& 4938& 6844&11063\\
 & 5719 & 4750 & 3906 & 9313 & 11219\\
 &  2531 & 9563 & 15281 & 12156 & 12969\\\midrule

\multirow{5}{*}{Breakout} & 20.09 &8.563 & 29.78&14.45 &35.38\\ 
 & 18.25 &25.03 & 27.56 &23.64 &20.91\\
  & 20.94 & 24.81& 0.625&19.69 &20.5\\
 &12.69 & 14.25 & 26.13 & 21.13 & 15.84 \\
 & 22.69 & 26.53 & 28 & 18.81 & 27.59\\\midrule

\multirow{5}{*}{CrazyClimber} & 54569 & 57534& 75300&69494 &55194\\ 
 & 51244 & 58522&74141& 59838&59934\\
  & 12959 & 62266&65431& 61503&70719\\
 & 47391 & 69391 & 47234 & 53328 & 68747\\
 &51128 & 50019 & 58847 & 50866 & 48984\\\midrule

\multirow{5}{*}{DemonAttack} & 55.31 & 215.3 & 134.1 &129.7 &169.1\\ 
 & 112.7 & 71.41& 155.5&105.9 &100.9\\
  & 127.7 & 152.2& 142.2&62.5 &166.4\\
 & 159.5 & 102 & 75.31 & 151.1 & 141.1 \\
 & 148.1 & 140.6 & 153 & 219.5 & 131.3 \\\midrule

\multirow{5}{*}{Frostbite} &261.3  &256.6& 250&263.4 &268.4\\ 
 &251.9  &241.6 &259.4& 258.1&249.4\\
  & 259.1 &242.2 &259.1& 267.5&268.4\\
 & 262.5 & 268.1 & 261.3 & 259.1 & 315.6 \\
 &  266.9 & 264.4 & 242.2 & 268.1 & 269.4\\\midrule

\multirow{5}{*}{JamesBond} & 268.8 & 12.5& 59.38& 371.9&232.8\\ 
 & 240.6 & 282.8&332.8&117.2 &101.6\\
  & 256.3 & 82.81&92.19& 23.44&228.1\\
 & 257.8 & 350 & 126.6 & 262.5 & 203.1 \\
 & 204.7 & 282.8 & 301.6 & 26.56 & 412.5 \\\midrule

\multirow{5}{*}{Kangaroo} & 987.5 &3294&293.8 & 25&1144\\ 
 & 56.25 &1588 &362.5& 1481&956.3\\
  & 112.5 & 37.5&175& 1719&1444\\
 & 37.5 & 5500 & 1756 & 1681 & 1663 \\
 & 1688 & 587.5 & 1600 & 1581 & 725 \\\midrule

\multirow{5}{*}{Krull} &5639  & 6124&5150& 5548&5569\\ 
 & 4873 & 3103&3290& 7266&4906\\
  & 2868 & 2142&4460& 1443&5744\\
 &7035 & 2384 & 5386 & 6236 & 3864 \\
 & 2244 & 1831 & 5806 & 5430 & 6591 \\\midrule

\multirow{5}{*}{Qbert} & 3002 &3935& 516.4& 1193&1082\\ 
 & 4151 & 1133&1420 &4190&3916\\
  & 4106 & 857&873.4& 3494&4208\\
 & 806.3 & 3198 & 1034 & 3325 & 3983 \\
 & 849.2 & 640.6 & 814.1 & 4462 & 631.3 \\\midrule

\multirow{5}{*}{RoadRunner} & 2793 & 8794&3034 &5709 &8666\\ 
 & 831.3 &8744 &4763&6763 &8541\\
  & 5034 & 6188&7397&12622 &9538\\
 & 3219 & 2791 & 8069 & 2581 & 9566 \\
 & 46.88 & 9375 & 1000 & 2675 & 2684 \\\midrule

\multirow{5}{*}{Seaquest} &392.5  &791.3&221.3 &649.4 &536.3\\ 
 & 419.4 & 286.3& 288.1&604.4&813.8\\
  & 395 & 692.5& 849.4&854.4&861.9\\
 & 249.4 & 760.6 & 851.3 & 513.8 & 671.9 \\
 & 151.9 & 691.3 & 832.5 & 598.8 & 203.8 \\
\bottomrule

\end{longtable}
}
\clearpage
{
\scriptsize

\begin{longtable}{@{} c c c @{}}
\caption{Scores achieved by every independent run of every SimPLe(30) and EVaDE-SimPLe agent when trained on the remaining 14 games of the Atari 100K test suite.} \label{tab:all_scores_14_game}
\\

\toprule
Game&SimPLe(30) & EVaDE-SimPLe\\  \midrule

\multirow{5}{*}{Alien} 
&579.1 & 671.9 \\
&494.1 & 605.9 \\
&387.5 & 545.9 \\
&330.3 & 444.1 \\
&33.44 & 595.6 \\ \midrule
\multirow{5}{*}{Amidar}
& 90.25 & 112.1 \\
& 29.56 & 128.7 \\
& 92.25 & 171.5 \\
& 84.41 & 99.25 \\
& 75.19 & 149.8 \\ \midrule
\multirow{5}{*}{Assault}
& 2170 & 1255 \\
& 895.2 & 868.1 \\
& 901.9 & 1315 \\
& 768.4 & 868.8 \\
& 961.5 & 837.8 \\ \midrule
\multirow{5}{*}{Asterix} 
& 1559 & 1228 \\
& 345.3 & 1322 \\
& 1269 & 1150 \\
& 1569 & 1727 \\
& 403.1 & 917.2 \\ \midrule
\multirow{5}{*}{Boxing} 
& 47.91 & 26.69 \\
&35.94 & 44.63 \\
&41.56 & 44.66 \\
&27.81 & 42.16 \\
&17.38 & 39.44 \\ \midrule
\multirow{5}{*}{ChopperCommand} 
& 859.4 & 734.4 \\
&878.8 & 984.8 \\
&487.5 & 953.1 \\
&821.9 & 818.2 \\
&872.7 & 875   \\ \midrule
\multirow{5}{*}{Freeway} 
& 33.63 & 33.94 \\
&20.84 & 32.5 \\
&32.47 & 33.66 \\
&33.44 & 33.31 \\
&33.78 & 33.41 \\ \midrule
\multirow{5}{*}{Gopher} 
& 659.4 & 752.5 \\
&293.8 & 793.1 \\
&690.6 & 1854 \\
&656.9 & 530 \\
&1038 & 423.8 \\ \midrule
\multirow{5}{*}{Hero} 
& 3028 & 3056 \\
&2908 & 2904 \\
&2976 & 3009 \\
&71.88 & 3275 \\
&3079 & 3004   \\ \midrule
\multirow{5}{*}{KungFuMaster}
& 13703 & 14175 \\
&17406 & 14384 \\
&13481 & 21191 \\
&11175 & 21684 \\
&13006 & 14248 \\ \midrule
\multirow{5}{*}{MsPacman} 
& 1194 & 1794 \\
&939.1 & 1551 \\
&1400 & 1050 \\
&1118 & 1483 \\
&1058 & 1688 \\ \midrule
\multirow{5}{*}{Pong} 
&4.313 & 6.375 \\
&-1.781 & 16.13 \\
&-17.78 & 13.03 \\
&-7.781 & 10.22 \\
&17.31 & 20.09 \\ \midrule
\multirow{5}{*}{PrivateEye} 
& 0 & 34.09 \\
&-38.25 & 100 \\
&332 & 0 \\
&4071 & 100 \\
&100 & 100\\ \midrule
\multirow{5}{*}{UpNDown} 
& 566.9 & 1452 \\
&1163 & 1182 \\
&1016 & 1681 \\
&1870 & 1586 \\
&236.6 & 1264 \\ \bottomrule
\end{longtable}
}

\clearpage    
\begin{table*}
\small
\begin{longtable}{@{}ccccccc@{}}
\caption{Mean scores achieved by SimPLe(30), EVaDE-SimPLe and other popular baselines in the Atari 100K test-suite.} \label{tab:all_mean_scores}
\\

\toprule
Game           & SimPLe  & SimPLe(30) & Curl    & OTRainbow & Eff Rainbow & EVaDE-Simple \\ \midrule
Alien          & 616.9   & 364.888    & 558.2   & \textbf{824.7}     & 739.9       & 572.68       \\
Amidar         & 88      & 74.332     & 142.1   & 82.8      & \textbf{188.6}       & 132.27       \\
Assault        & 527.2   & \textbf{1139.4}     & 600.6   & 351.9     & 431.2       & 1028.94      \\
Asterix        & 1128.3  & 1029.08    & 734.5   & 628.5     & 470.8       & \textbf{1268.84}      \\
BankHeist      & 34.2    & 78.557     & 131.6   & 182.1     & 51          & \textbf{224.16}       \\
BattleZone     & 5184.4  & 4543.8     & \textbf{14870}   & 4060.6    & 10124.6     & 11094        \\
Boxing         & 9.1     & 34.12      & 1.2     & 2.5       & 0.2         & \textbf{39.516}       \\
Breakout       & 16.4    & 18.932     & 4.9     & 9.84      & 1.9         & \textbf{24.024}       \\
ChopperCommand & \textbf{1246.9}  & 784.06     & 1058.5  & 1033.33   & 861.8       & 873.1        \\
CrazyClimber   & \textbf{62583.6} & 43458.2    & 12146.5 & 21327.8   & 16185.3     & 60715.6      \\
DemonAttack    & 208.1   & 120.662    & \textbf{817.6}   & 711.8     & 508         & 141.76       \\
Freeway        & 20.3    & 30.832     & 26.7    & 25        & 27.9        & \textbf{33.364}       \\
Frostbite      & 254.7   & 260.34     & \textbf{1181.3}  & 231.6     & 866.8       & 274.24       \\
Gopher         & 771     & 667.74     & 669.3   & 778       & 349.5       & \textbf{870.68}       \\
Hero           & 2656.6  & 2412.576   & 6279.3  & 6458.8    & \textbf{6857}        & 3049.6       \\
Jamesbond      & 125.3   & 245.64     & \textbf{471}     & 112.3     & 301.6       & 235.62       \\
Kangaroo       & 323.1   & 576.35     & 872.5   & 605.4     & 779.3       & \textbf{1186.46}      \\
Krull          & 4539.9  & 4531.8     & 4229.6  & 3277.9    & 2851.5      & \textbf{5334.8}       \\
KungFuMaster   & \textbf{17257.2} & 13754.2    & 14307.8 & 5722.2    & 14346.1     & 17136.4      \\
MsPacman       & 1480    & 1141.82    & 1465.5  & 941.9     & 1204.1      & \textbf{1513.2}       \\
Pong           & 12.8    & -1.1438    & -16.5   & 1.3       & -19.3       & \textbf{13.169}       \\
PrivateEye     & 58.3    & \textbf{892.95}     & 218.4   & 100       & 97.8        & 66.818       \\
Qbert          & 1288.8  & 2582.9     & 1042.4  & 509.3     & 1152.9      & \textbf{2764.06}      \\
RoadRunner     & 5640.6  & 2384.836   & 5661    & 2696.7    & \textbf{9600}         & 7799         \\
Seaquest       & \textbf{683.3}   & 321.64     & 384.5   & 286.92    & 354.1       & 617.54       \\
UpNDown        & \textbf{3350.3}  & 970.5      & 2955.2  & 2847.6    & 2877.4      & 1433         \\
\bottomrule

\end{longtable}
\end{table*}

\begin{figure*}[htbp]
\centering
{
\scriptsize
\hspace*{\fill}
\subfigure{
\includegraphics[width=0.28\linewidth]{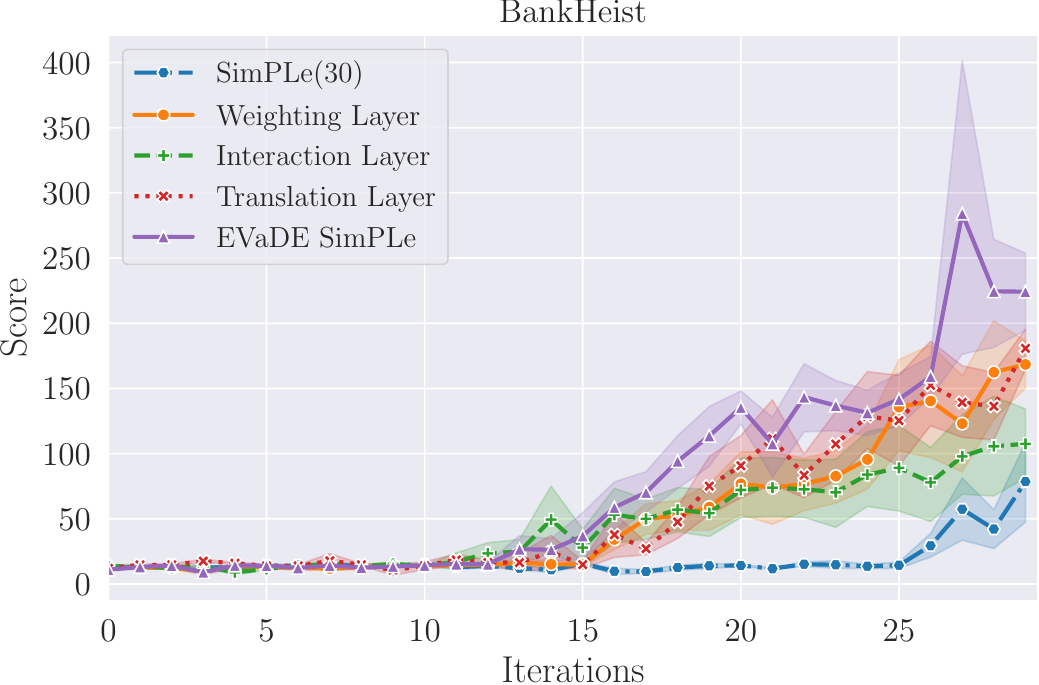}
}
\hfill
\subfigure{
\includegraphics[width=0.28\linewidth]{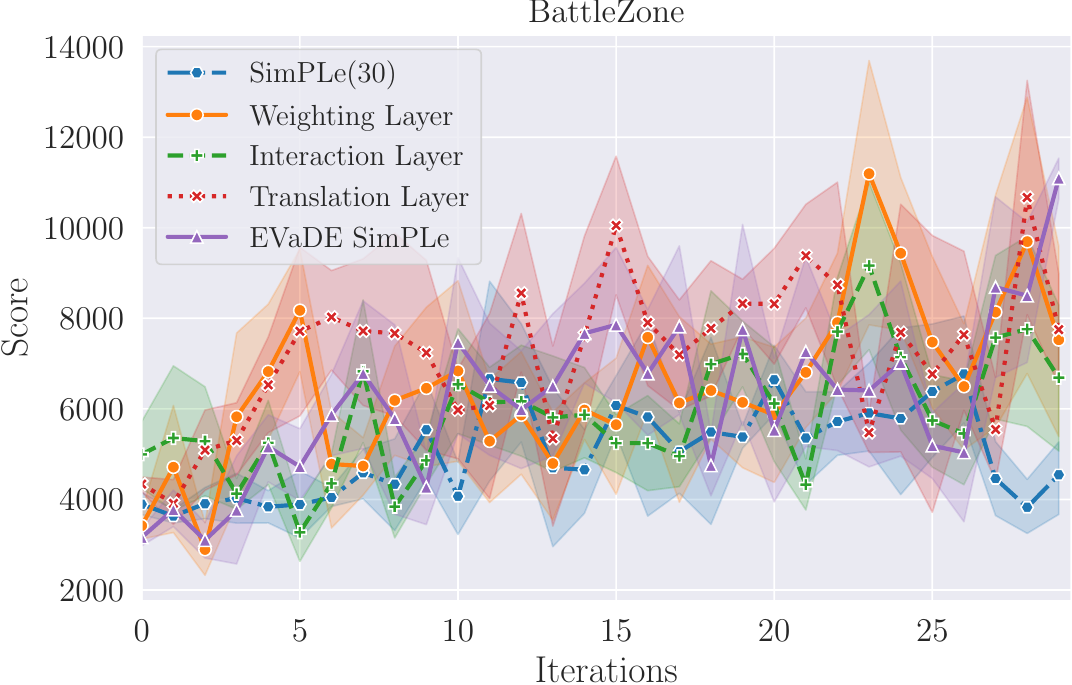}
} 
\hfill
\subfigure{
\includegraphics[width=0.28\linewidth]{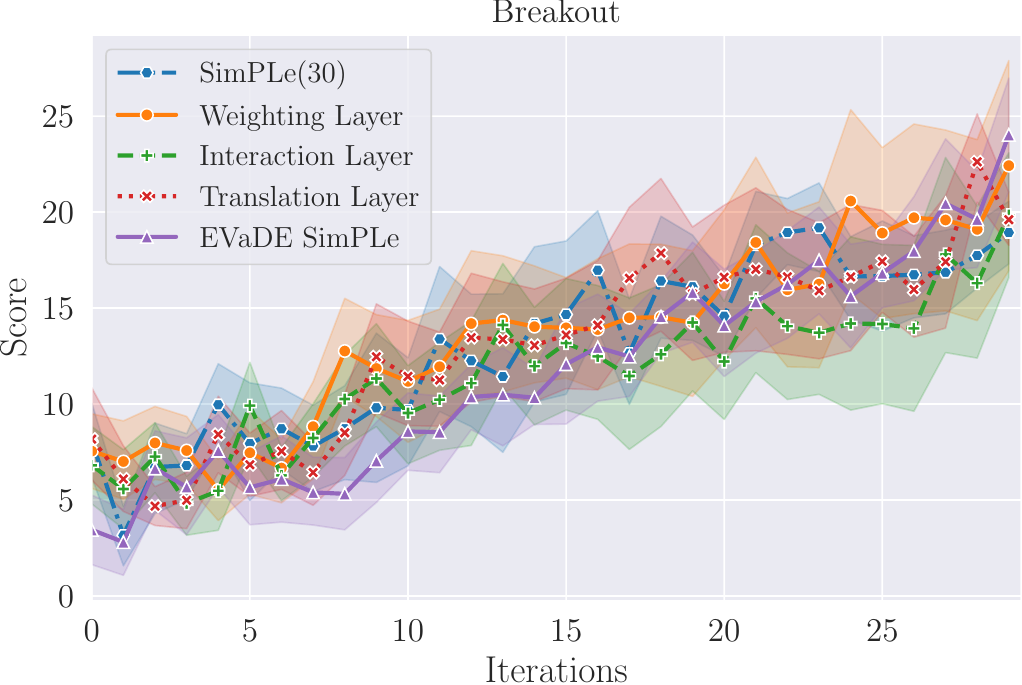}
}
\hspace*{\fill}
\newline 
\hspace*{\fill}
\subfigure{
\includegraphics[width=0.28\linewidth]{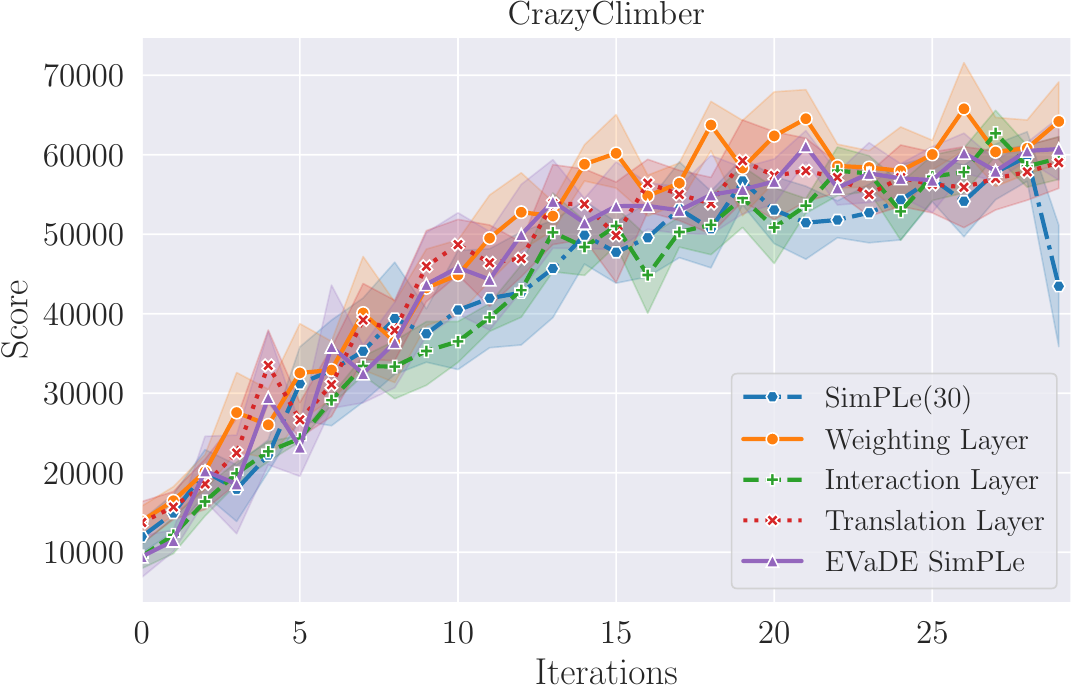}
}
\hfill
\subfigure{
\includegraphics[width=0.28\linewidth]{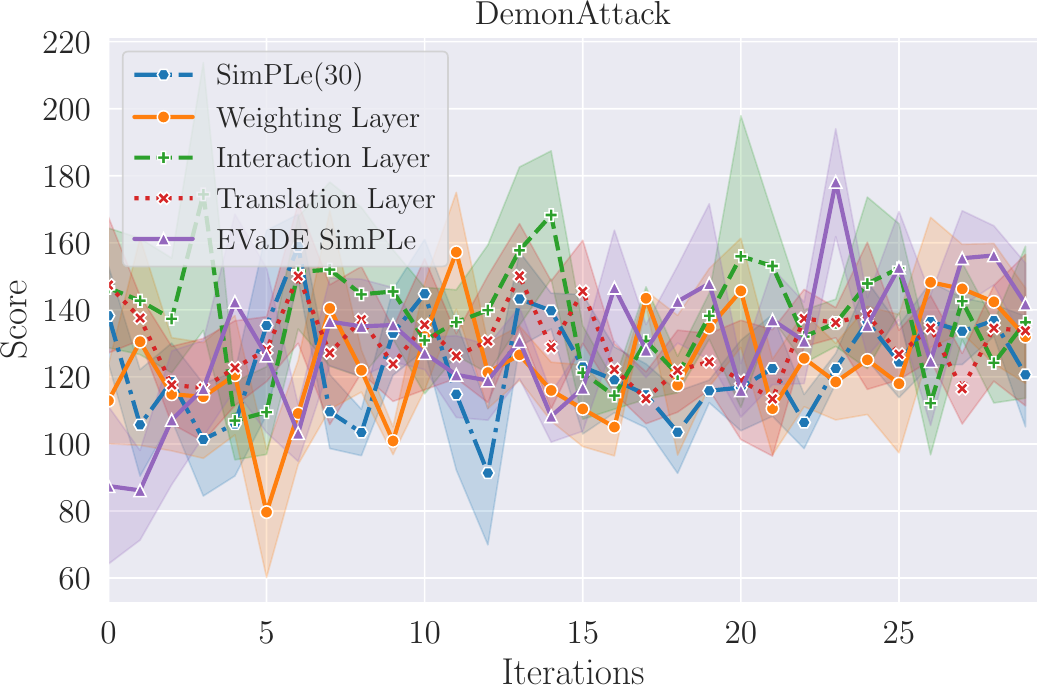}
}
\hfill
\subfigure{
\includegraphics[width=0.28\linewidth]{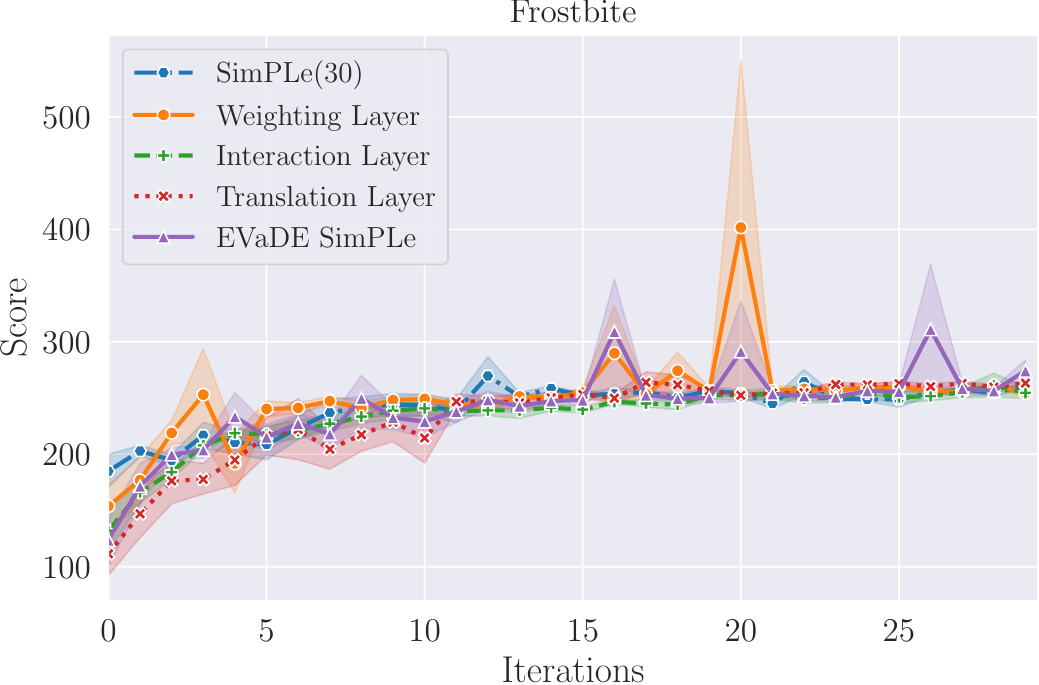}
}
\hspace*{\fill}
\newline 
\hspace*{\fill}
\subfigure{
\includegraphics[width=0.28\linewidth]{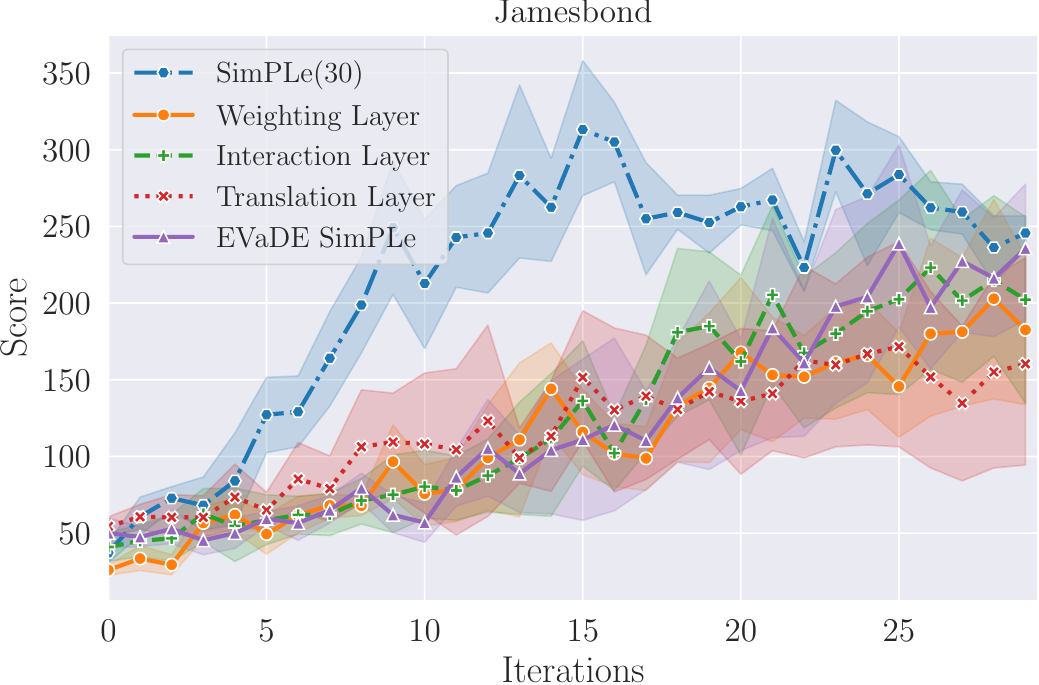}
}
\hfill
\subfigure{
\includegraphics[width=0.28\linewidth]{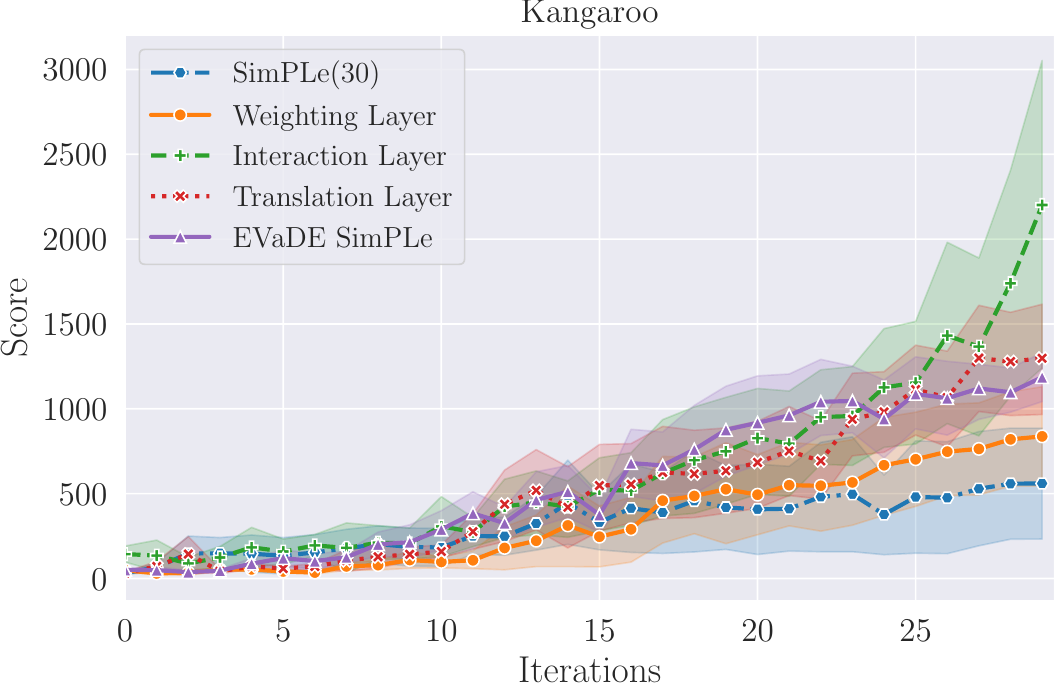}
}
\hfill
\subfigure{
\includegraphics[width=0.28\linewidth]{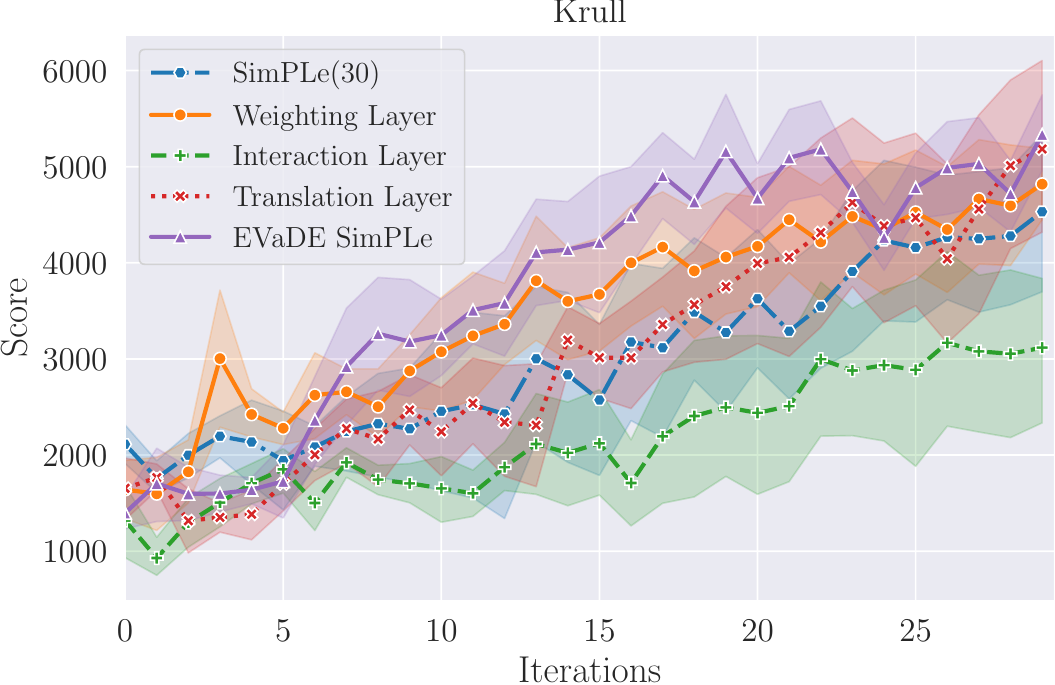}
}
\hspace*{\fill}
\newline
\hspace*{\fill}
\subfigure{
\includegraphics[width=0.28\linewidth]{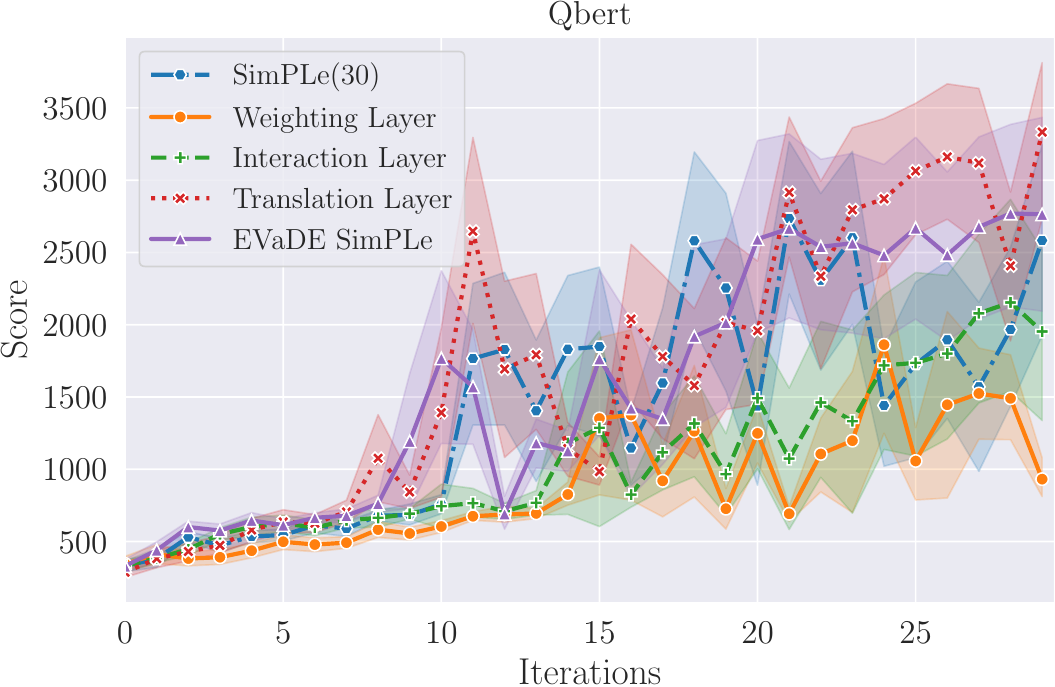}
}
\hfill
\subfigure{
\includegraphics[width=0.28\linewidth]{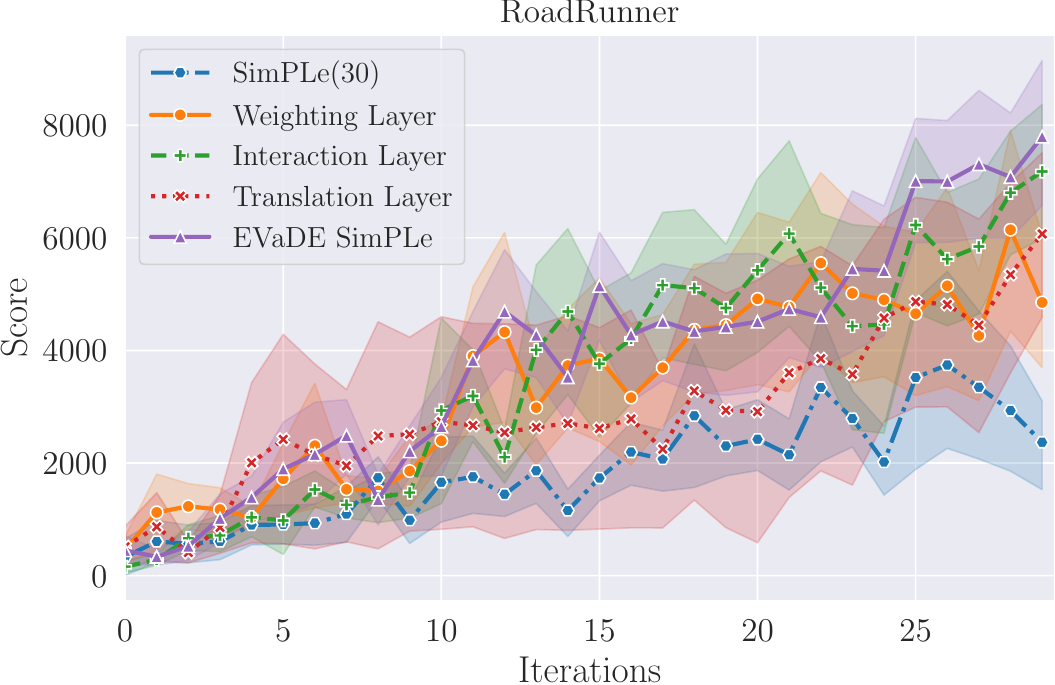}
}
\hfill
\subfigure{
\includegraphics[width=0.28\linewidth]{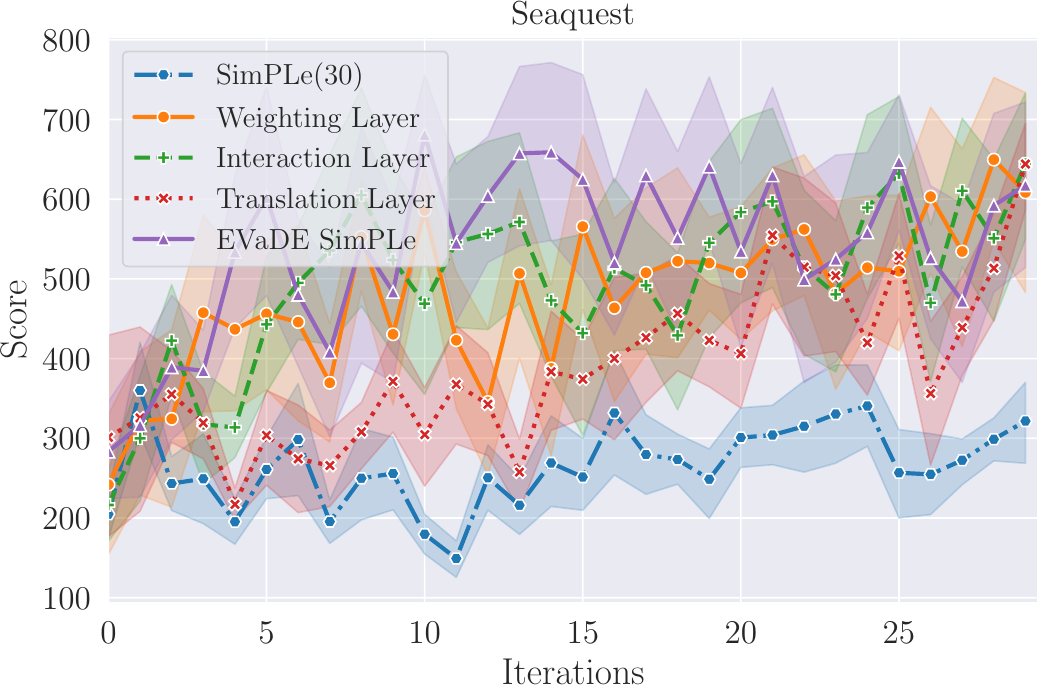}
}
\hspace*{\fill}
}
\caption{Learning curves of EVaDE-SimPLe agents, SimPLe(30) agents and agents which only add one of the EVaDE layers with error bars of 1 standard error.}
\label{fig:evade_exp2_err}
\end{figure*} 

\clearpage